\documentclass[journal]{IEEEtran}
\usepackage{caption}
\usepackage{graphicx}
\usepackage{algorithm}
\usepackage{algorithmicx}
\usepackage{algpseudocode}
\usepackage{makecell}
\usepackage{subfigure,color,balance}
\usepackage{multirow}
\usepackage{makecell}
\usepackage{bm}
\usepackage{diagbox}
\usepackage{amsmath} 
\usepackage{amssymb}  
\usepackage{diagbox}

\usepackage{lineno}
\usepackage[colorlinks,
linkcolor=black,
anchorcolor=black,
citecolor=black
]{hyperref}
\usepackage{graphics} 
\usepackage{epsfig} 
\usepackage{times} 
\usepackage{amsfonts}
\usepackage{epstopdf}
\usepackage{verbatim}
\usepackage{cases}
\usepackage{enumerate}
\usepackage[shortlabels]{enumitem}
\usepackage{balance}
\usepackage{framed}
\usepackage{booktabs}
\usepackage{color}
\usepackage{cancel}
\usepackage{epstopdf}
\usepackage{booktabs}

\graphicspath{{figures/}}

\newcommand{\ie}{\textit{i}.\textit{e}.}
\newcommand{\eg}{\textit{e}.\textit{g}.}

\newtheorem{definition}{Definition}
\newtheorem{remark}{Remark}
\newtheorem{lemma}{Lemma}
\newtheorem{theorem}{Theorem}
\newtheorem{example}{Example}
\newtheorem{assumption}{Assumption}
\newtheorem{proposition}{Proposition}
\ifCLASSINFOpdf
\else
\fi
\begin{document}
%
\title{Conflict-free Cooperation Method for Connected and Automated Vehicles at Unsignalized Intersections: Graph-based Modeling and \\Optimality Analysis}
%
%
%
\author{
	Chaoyi Chen$^{1}$,
	Qing Xu$^{*,1}$,
	Mengchi Cai$^{1}$,
	Jiawei Wang$^{1}$,
	Jianqiang Wang$^{1}$, 
	Keqiang Li$^{1}$.
	\thanks{This work was supported by the National Key Research and Development Program of China under Grant 2018YFE0204302, National Natural Science Foundation of China under Grant 52072212, Key Algorithms in Intelligent and Connected Cloud Control System for China Intelligent and Connected Vehicles Research Institute (CICV), Tsinghua University-Didi Joint Research Center for Future Mobility and Tsinghua-Toyota Joint Research Institute Cross-discipline Program..}
	\thanks{$^{1}$Chaoyi Chen, Qing Xu, Mengchi Cai, Jianqiang Wang and Keqiang Li are with School of Vehicle and Mobility, Tsinghua University}
	\thanks{*Corresponding author: Qing Xu, Email address:qingxu@tsinghua.edu.cn}
}

%
%

\markboth{IEEE Transactions on Intelligent Transportation Systems,~Vol.~14, No.~8, April~2022}%
{}

%



\maketitle

\begin{abstract}
Connected and automated vehicles have shown great potential in improving traffic mobility and reducing emissions, especially at unsignalized intersections. Previous research has shown that vehicle passing order is the key influencing factor in improving intersection traffic mobility. In this paper, we propose a graph-based cooperation method to formalize the conflict-free scheduling problem at an unsignalized intersection. Based on graphical analysis, a vehicle's trajectory conflict relationship is modeled as a conflict directed graph and a coexisting undirected graph. Then, two graph-based methods are proposed to find the vehicle passing order. The first is an improved depth-first spanning tree algorithm, which aims to find the local optimal passing order vehicle by vehicle. The other novel method is a minimum clique cover algorithm, which identifies the global optimal solution. Finally, a distributed control framework and communication topology are presented to realize the conflict-free cooperation of vehicles. Extensive numerical simulations are conducted for various numbers of vehicles and traffic volumes, and the simulation results prove the effectiveness of the proposed algorithms.
\end{abstract}

\begin{IEEEkeywords}
Connected and Automated Vehicle, Unsignalized intersection, Conflict-free cooperation, Distributed control  
\end{IEEEkeywords}

%
\IEEEpeerreviewmaketitle

\section{INTRODUCTION}
	Intersections are the most common merging points in urban traffic scenarios~\cite{xu2021coordinated}. According to the Federal Highway Administration, more than 2.8 million intersection-related crashes occur in the US each year, accounting for 44\% of all crashes~\cite{azimi2014stip}. The rapid development of the V2X technology provides an opportunity to improve traffic mobility and safety in intersection management~\cite{contreras2017internet}. {Through vehicle-to-vehicle (V2V) and vehicle-to-infrastructure (V2I) communication, a centralized controller is deployed at the intersection to coordinate the connected and automated vehicles (CAVs) to pass through the intersection; this guarantees conflict-free vehicle cooperation.} Hence, CAVs have immense potential in improving traffic safety and mobility at intersections~\cite{zheng2020smoothing, chen2021mixed}.

In a fully autonomous intersection scenario, a hierarchical framework is frequently used to realize the CAV coordination~\cite{xu2017v2i}. First, after collecting the information of the CAVs in real-time, the centralized controller schedules the arrival time of the CAVs to improve the traffic efficiency. Then, a distributed controller is applied to the CAV to implement the determined arrival plan. Multiple methods have been proposed to address the vehicle control problem,~\eg, fuzzy logic~\cite{milanes2009controller, onieva2012genetic}, model predictive control~\cite{du2018hierarchical, he2015optimal}, and optimal control~\cite{malikopoulos2018decentralized, zhang2017decentralized}. Moreover, the virtual platoon method was used in~\cite{xu2018distributed,chavoshi2021pairing}, which projects the CAVs from different lanes into a virtual lane. Thus, the typical vehicle platoon analysis methods~\cite{zheng2015stabilitymargin} can be used in the control problem of the CAVs on different lanes as if they are traveling in the same lane.


{Intersections are the converging point of traffic flows in different directions. The intersection geometry structures regulate the routes of incoming CAVs, which leads to complicated conflict relationships at intersections. Therefore, CAVs' arrival time needs to be staggered to avoid collisions. What's more, a high-efficiency conflict-free arrival plan is required to improve traffic mobility.} Researchers have also pointed out that the passing order of the vehicles is the key factor influencing traffic mobility at intersections~\cite{li2006cooperative, meng2017analysis}. As incoming CAVs approach the intersection from different directions, the optimal passing order changes constantly and the solution space increases exponentially, preventing the determination of the optimal passing order using the brute-force method. This problem is typically solved using a reservation-based method~\cite{dresner2004multiagent}. The most straightforward reservation-based solution is the first-in-first-out (FIFO) strategy, in which vehicles that enter the intersection first are scheduled to leave it first~\cite{xu2017v2i, malikopoulos2018decentralized, dresner2008multiagent}. However, the FIFO passing order is not likely to be the optimal solution in most cases. Batch-based strategy is an improved version of the FIFO strategy, and it processes the vehicles in batches according to their direction to reduce the traffic delay~\cite{tachet2016revisiting}; however, its performance has not been fully optimized. Another widely used strategy is the optimization-based method. After formulating the scheduling problem as an optimization problem, various methods have been proposed to find the optimal passing order,~\eg, mixed-integer programming~\cite{mirheli2019consensus}, Monte Carlo tree search~\cite{xu2019cooperative}, and dynamic programming~\cite{yan2009autonomous}. However, the solution space exponentially increases with the number of the vehicles, making the use of optimization-based methods unfeasible.


{
Apart from these specific methods, graph theory is also introduced to find the optimal passing order of multiple CAVs, which is a discrete problem. For example,~\cite{pandit2013adaptive, dey2013fuzzy} employs conflict graph analysis to optimize the traffic-signal phase plan. In terms of multiple CAVs scheduling, several graph-based methods have also been used. Apart from the depth-first spanning tree (DFST) algorithm proposed in~\cite{xu2018distributed}, Petri net~\cite{lin2019graph} and conflict duration graph~\cite{deng2020conflict} are also used in modeling the scheduling problem. However, the discussion on the optimality of the scheduling problem of multiple CAVs is inadequate.
}

To tackle the problem of optimality, mathematical modeling,~\ie, the computational reduction of the intersection scheduling problem, is an essential research aspect. Some researchers have focused on reducing this problem into typical algorithmic problems, including the job-shop scheduling problem~\cite{ahn2017safety}, abstraction-based verification problem~\cite{ahn2019abstraction}, and polling system problem~\cite{miculescu2019polling}. However, most of these methods mainly focus on the feasibility of a conflict-free passing order solution. In practice, traffic efficiency is another critical topic that has not been fully considered in the computational reduction of this problem.

{
In this study, we mainly focused on finding the optimal CAV passing order at unsignalized intersections,~\ie, the CAV scheduling problem. Most of the existing studies have focused on obtaining the passing order through specific methods without considering the optimality and computational complexity~\cite{xu2018distributed, li2006cooperative, meng2017analysis, dresner2008multiagent, xu2019cooperative}. Some have focused on the feasibility of a conflict-free solution but neglected traffic efficiency optimization~\cite{tachet2016revisiting, lin2019graph, deng2020conflict, ahn2017safety, ahn2019abstraction, miculescu2019polling}. To address this problem, we introduce a graph-based conflict-free cooperation method to model the conflict relationship of the vehicles. Based on the conflict analysis, two novel methods are proposed to obtain the optimal vehicle passing order. Improved DFST (iDFST) obtains the local optimal solution while minimum clique cover (MCC) targets the global optimal one. The contributions of our study are as follows:

\begin{enumerate}
	\item An improved DFST method is proposed based on~\cite{xu2018distributed} to improve traffic efficiency. Instead of considering all conflict types as one united conflict set, we categorize different conflict types to further reduce the overall depth of the spanning tree,~\ie, the evacuation time of vehicles. The computational complexity of the improved algorithm remains low, and the passing order solution generated by it is proved to be the local optimum.
	
	\item Based on the concise and rigid graphical analysis, we first reduce the CAV scheduling problem to the minimum clique cover problem. Unlike most of the existing research~\cite{malikopoulos2018decentralized, xu2018distributed, dresner2008multiagent}, our reduction of the problem is proposed without the assumption of the FIFO principle. Hence, the global optimal passing order solution is obtained by solving the MCC problem. Moreover, we propose a heuristic method to solve the problem of numerous vehicles with a low computation burden.
	
	\item A distributed feedback control method is designed to apply the scheduling results of the proposed algorithms. In particular, the predecessor--leader following topology is used as the communication topology to lower the communication bandwidth requirement. Traffic simulations are conducted for various numbers of vehicle inputs and traffic volumes. The simulation results reveal that the proposed algorithms significantly improve both traffic efficiency and fuel economy with acceptable computational time.
\end{enumerate}
}

The remainder of this paper is organized as follows. Section~\ref{Sec:Scenario} introduces the scenario addressed in this study. Section~\ref{Sec:Methodology} presents the conflict analysis, iDFST method, MCC method, and distributed control. The simulation results are presented in Section~\ref{Sec:Simulation}, and Section~\ref{Sec:Conclusion} concludes the paper.

\section{Problem Statement}
	\label{Sec:Scenario}
	{Figure}~\ref{fig:Traffic_Scenario_0} depicts one of the most congested intersections near the Tsinghua campus. In this scenario, because the number of approaching lanes is higher than the number of departure lanes, traffic congestion frequently occurs during rush hour. In this study, we simplified the traffic scenario by separating the vehicles into different lanes based on their destination. A hierarchical framework was established to realize the cooperation of the CAVs~\cite{xu2017v2i}. A \emph{central cloud coordinator} was deployed at the intersection to guide the CAVs in driving through the intersection. After collecting the positions and velocities of all the CAVs, the upper coordinator schedules the passing order of the approaching CAVs, while the lower distributed vehicle controller controls the vehicle according to the determined passing order. The scheduling method is presented in Sections~\ref{Sec:ConflictAnalysis},~\ref{Sec:iDFST},~\ref{Sec:MCC}, and the distributed control is presented in Section~\ref{Sec:DistributedControl}. First, the following assumptions were devised to facilitate the design of the strategy.
	
	{
	\begin{assumption}
		\label{asp:LaneChange}
		Similar to most of the existing studies on intersection traffic~\cite{malikopoulos2018decentralized, xu2018distributed, xu2019cooperative}, lane change behavior was prohibited to guarantee vehicle safety and improve traffic efficiency. For safety reasons, vehicles with conflict relationships are not allowed to drive in the conflict zone of the intersection simultaneously.
	\end{assumption}
	}
	
	\begin{assumption}
		An ideal communication condition without communication delay or packet loss is under consideration. CAVs transmit their velocity and position to the central cloud coordinator through wireless communication, ~\eg, V2I communication~\cite{gerla2014internet}. 
	\end{assumption}
	
	\begin{assumption}
		The CAVs are capable of fully autonomous driving and are assumed to have perfect steering performance. Therefore, we only focused on the longitudinal control of the vehicle.
	\end{assumption}

	\begin{figure*}[!t]
		\centering
		\subfigure[Intersection Scenario\label{fig:Traffic_Scenario_0}]
		{\includegraphics[width=0.45\linewidth]{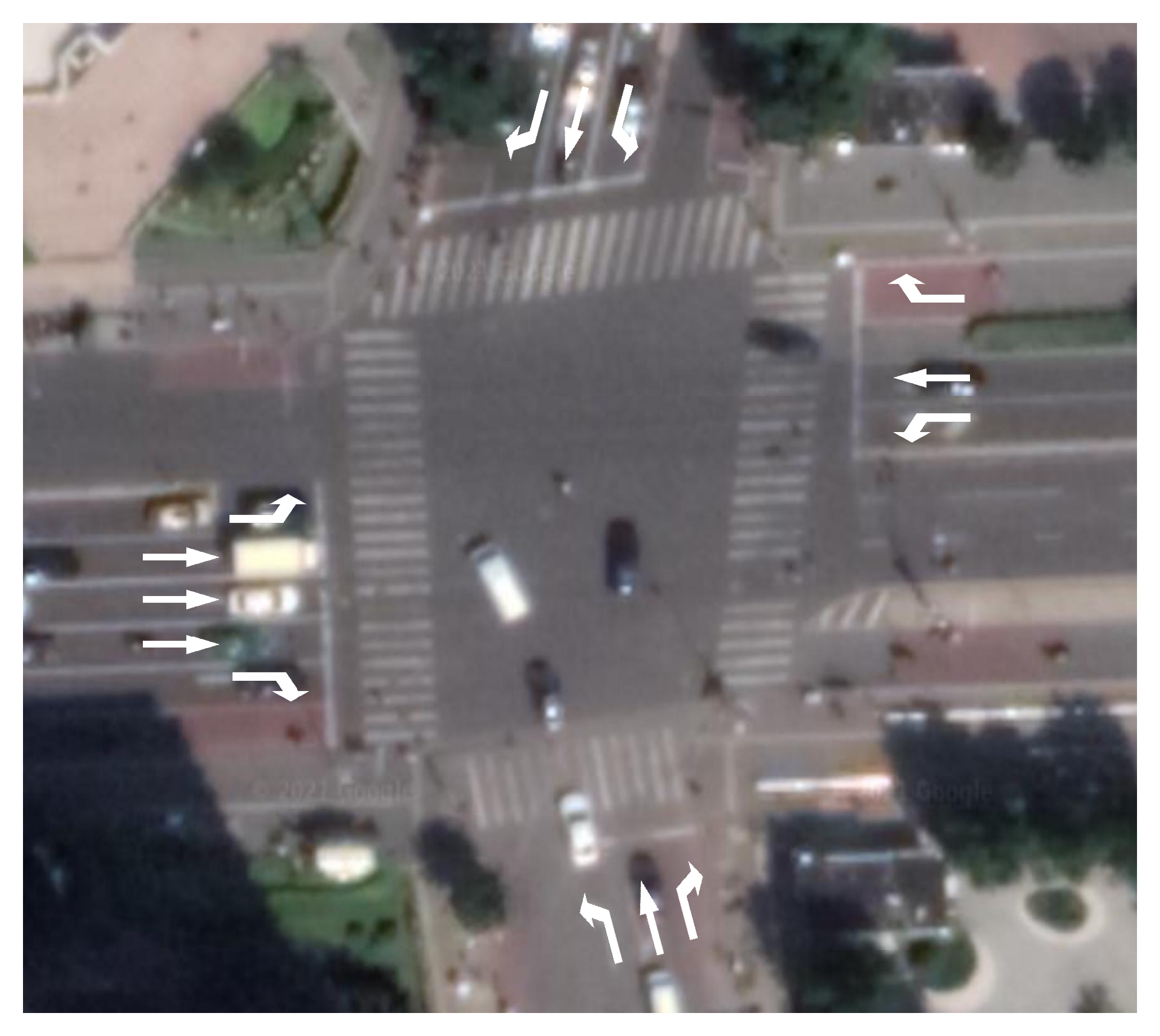}}
		\subfigure[Incoming Vehicles and Conflict Points\label{fig:Traffic_Scenario}]
		{\includegraphics[width=0.4\linewidth]{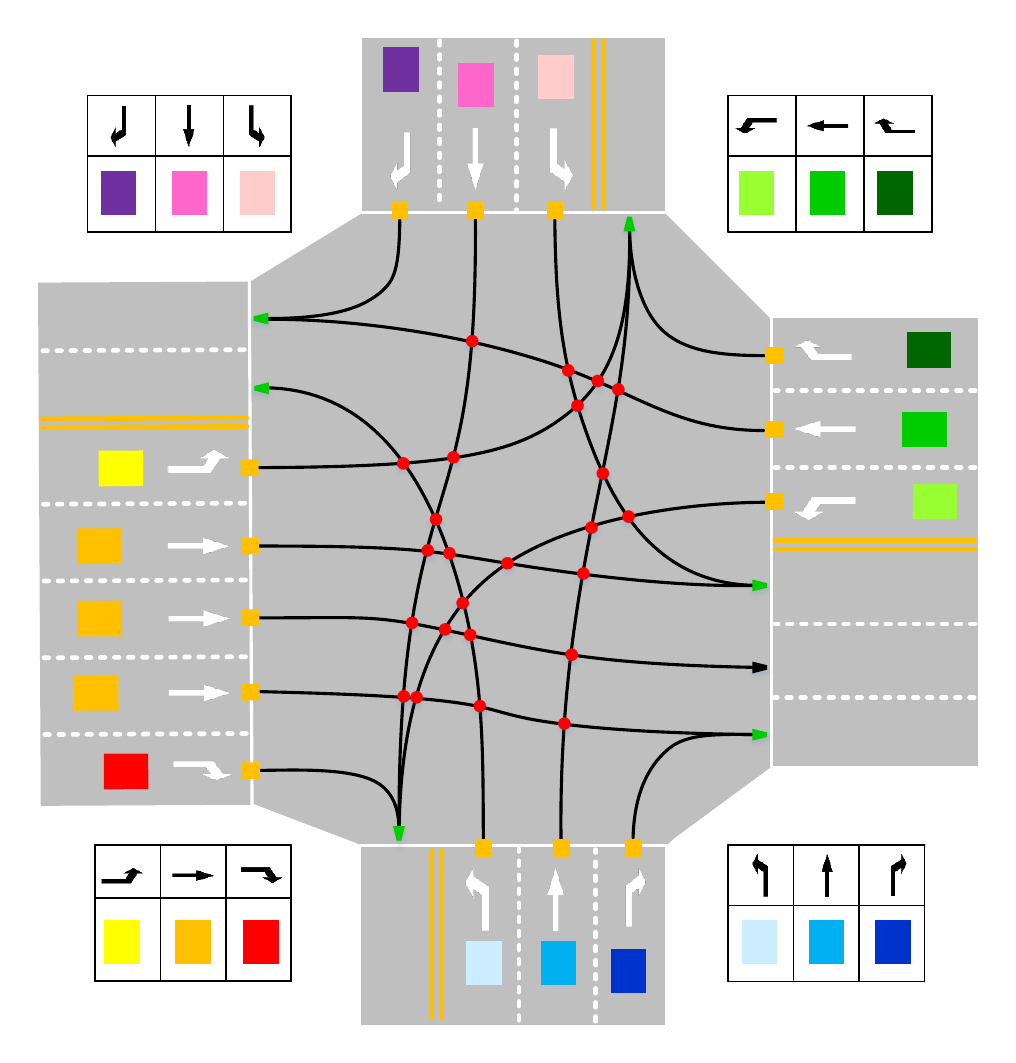}}
		\caption{One of the most congested intersections near the Tsinghua campus in Beijing is depicted in~\ref{fig:Traffic_Scenario_0} and the conflict relationship of the traffic scenario is illustrated in ~\ref{fig:Traffic_Scenario}. Vehicles coming from different lanes are distinguished by different colors. The red circles, orange squares, and green arrows represent different potential collision points, which are interpreted in Section~\ref{Sec:ConflictAnalysis}.}
		\label{fig:Intersection}
	\end{figure*}
	
	{
	The incoming CAVs are indexed from $ 1 $ to $ N $ according to their arrival sequence at the control zone of length $ L_\mathrm{ctrl} $ in Fig.~\ref{fig:AllConflicts},~\ie, only the CAVs in the control zone are under coordination. Note that since we focus on multi-vehicle scheduling problem rather than vehicle control problem, second-order vehicle model as defined in~\eqref{equ:StateSpaceEquation} is used instead of more complex vehicle dynamic model. For each CAV $ i \left(i \leq N, i \in \mathbb{N}^+\right) $, the CAV dynamic model is described as follows:
	}
	\begin{equation}
		\label{equ:StateSpaceEquation}
		\begin{array}{l}
			\dot{\boldsymbol{x}}_{i}(t)=\boldsymbol{A} \boldsymbol{x}_{i}(t)+\boldsymbol{B} u_{i}(t), \\
			\boldsymbol{x}_{i}(t)=\left[\begin{array}{c}
				-p_{i} \\
				v_{i} \\
			\end{array}\right], \quad \boldsymbol{A}=\left[\begin{array}{cc}
				0 & 1 \\
				0 & 0 \\
			\end{array}\right], \quad \boldsymbol{B}=\left[\begin{array}{c}
				0 \\
				1 \\
			\end{array}\right],
		\end{array}
	\end{equation}
	where $ \boldsymbol{x}_{i} $ is the state space of CAV $ i $, and $ {p}_{i}$, ${v}_{i}$, and $u_{i}(t) $ denote the remaining distance to the stopping line, velocity, and control input of CAV $ i $ at time $ t $. We consider homogeneous vehicle platoons, and thus a continuous-time linear time-invariant system, $ (\boldsymbol{A}, \boldsymbol{B}) $, is used.
	
	{
	There also exist velocity and acceleration constraints of the vehicle.
	\begin{equation}
		\label{equ:Constraints}
		\begin{aligned}
			0 &\le v_{i} \le v_{\max}, \\
			u_{\min} &\le u_{i} \le u_{\max},
		\end{aligned}
	\end{equation}
	where $ v_{\max} $ is the maximum velocity limit; $ u_{\min} $ and $ u_{\max} $ are the deceleration and acceleration limits, respectively.
	}
	
	As mentioned earlier, we aim to propose a cooperation method to improve both traffic safety and efficiency,~\ie, obtain the collision-free optimal CAV passing order. Assign $ t_{i}^{\mathrm{in}} $ as the time step when vehicle $ i $ enters the control zone, and $ t_{i}^{\mathrm{out}} $ as the time step when it arrives at the intersection. {Several performance indexes have been proposed to measure the scheduling performance. We use two performance indexes to represent overall traffic efficiency and average traffic efficiency respectively. The first one is overall traffic efficiency, namely evacuation time as shown in Definition~\ref{def:EvacuationTime}.}
	\begin{definition}[Evacuation Time]
		\label{def:EvacuationTime}
		The evacuation time of $ N $ CAVs is defined as the time when the last CAV reaches the stopping line; it is expressed as
		\begin{equation}
			\label{equ:EvacuationTime}
			t_\mathrm{evc} = \max t_{i}^{\mathrm{out}}, i \le N, i \in \mathbb{N}^+.
		\end{equation}
		Considering $ N $ incoming CAVs, $ t_\mathrm{evc} $ represents the arrival time of the last CAV at the stopping line. For $ N $ CAVs, smaller evacuation time means these CAVs pass through the intersection in shorter time. Thus, it shows the overall traffic efficiency performance,~\ie,the overall benefits of the CAVs. 
	\end{definition}
	
	{
	In addition, the vehicle travels through the control zone in $ t_{i}^{\mathrm{out}} - t_{i}^{\mathrm{in}} $ time, whereas it travels through it under the free driving condition in $ {L_{\mathrm{ctrl}}}/{v_{\max}} $ time. Accordingly, the average travel time delay (ATTD) is also chosen as performance index to measure the average traffic efficiency of the vehicles, as described in Definition~\ref{def:ATTD}.}
	\begin{definition}[Average Travel Time Delay]
		\label{def:ATTD}
		The ATTD is designed to evaluate the average traffic efficiency of $ N $ CAVs. It is expressed as
		\begin{equation}
			\label{equ:TravelTimeDelay}
			t_\mathrm{ATTD} = \frac{1}{N} \sum_{i = 1}^{N} \left(t_{i}^{\mathrm{out}} - t_{i}^{\mathrm{in}}-\frac{L_{\mathrm{ctrl}}}{v_{\max}}\right),
		\end{equation}
		where $ L_{\mathrm{ctrl}} $ is the length of the control zone and $ t_\mathrm{ATTD} $ represents the average travel delay of the CAVs. Since the travel time of every CAV is considered in $ t_\mathrm{ATTD} $, it denotes the individual benefits of $ N $ CAVs, which is the secondary optimization target of traffic efficiency.
	\end{definition}
	
	In this study, we mainly focused on finding the optimal CAV passing order to minimize the evacuation time, as defined in Definition~\ref{def:EvacuationTime}. In addition, the passing order should also satisfy the conflict-free relationship between the CAVs. Section~\ref{Sec:Methodology} presents the specific methods used to guarantee safety and improve traffic efficiency. Moreover, the ATTD defined in Definition~\ref{def:ATTD} is introduced to evaluate the algorithm performance. The simulation results and algorithm performance are presented in Section~\ref{Sec:Simulation}.

\section{Methodology}
\label{Sec:Methodology}
In the following paragraph, Section~\ref{Sec:ConflictAnalysis} first introduces the graph-based intersection conflict analysis method. Then, Section~\ref{Sec:iDFST} presents the optimization of the DFST method to obtain the local optimal passing order solution, termed as the iDFST method. In Section~\ref{Sec:MCC}, we introduce the MCC method to find the global optimal passing order solution. Finally, the distributed controller design is presented in Section~\ref{Sec:DistributedControl} to undertake the scheduling plan.

\subsection{Conflict Analysis}
\label{Sec:ConflictAnalysis}
{
Conflict analysis is one of the most fundamental area in multi-agent cooperation research~\cite{roess2004traffic, li2021design, li2022cooperative}. In~\cite{xu2018distributed}, the DFST method using a \emph{virtual platoon} was proposed, which projects the CAVs from different lanes onto a virtual lane to avoid collision. DFST scheduling method is built on vehicle conflict relationship analysis.
}

In Fig.~\ref{fig:Intersection}, the CAVs from different lanes form multiple conflict points. Despite the complicated conflict scenarios, they can be classified into the following conflict modes. Without loss of generality, several CAVs are selected in Fig.~\ref{fig:AllConflicts} to illustrate the conflict relationship. First, we introduce the route conflicts, where CAVs have trajectory intersections on their paths.

\begin{figure}[!t]
	\centering
	{\includegraphics[width=\linewidth]{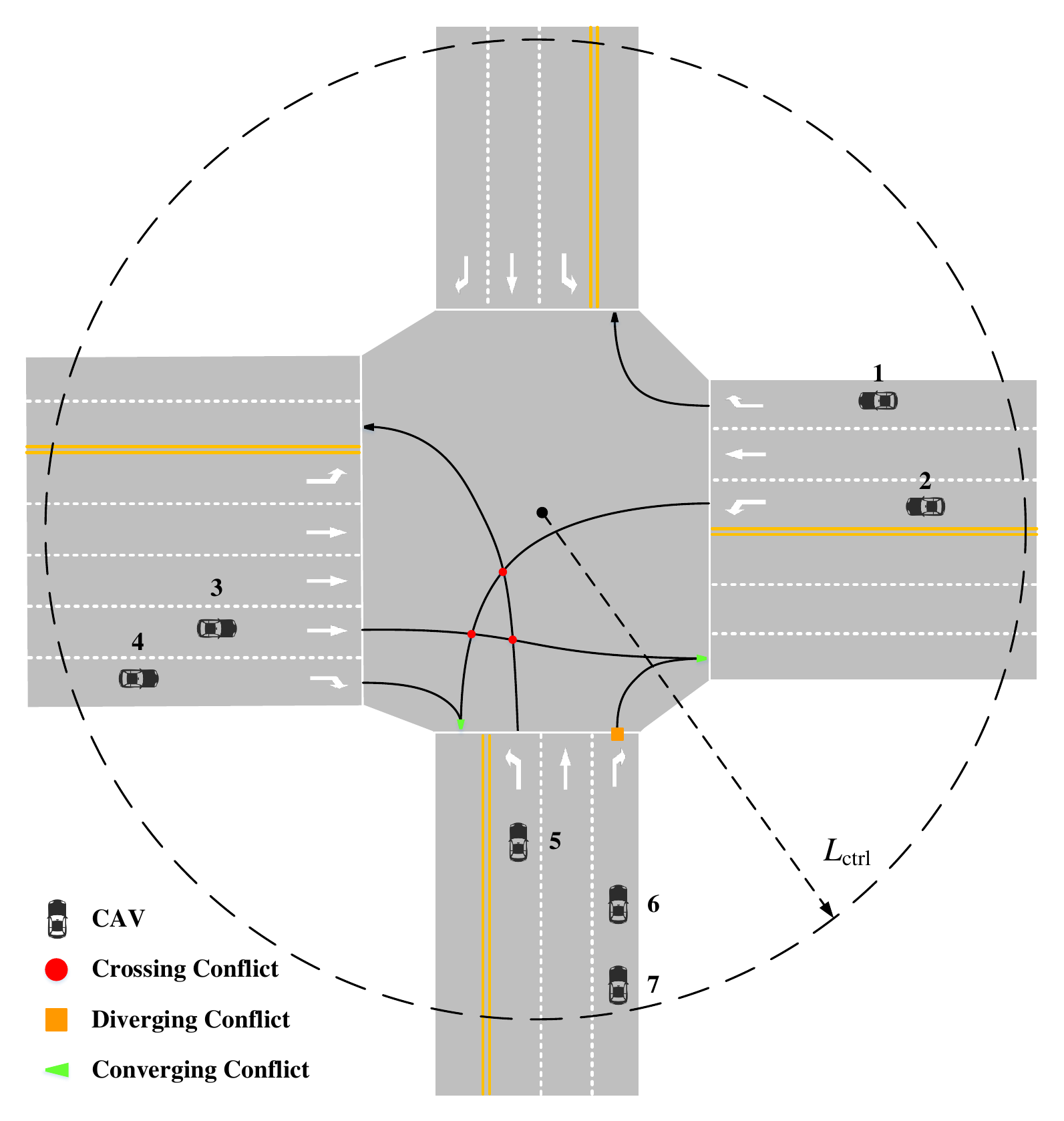}}
	\caption{Illustration of different conflict relationship between the seven CAVs. Crossing conflicts are plotted as red circles, diverging conflicts as yellow squares, and converging conflicts as green arrows.}
	\label{fig:AllConflicts}
\end{figure}



\begin{enumerate}[1)]
	\item \emph{Crossing Conflict}: Vehicles from different lanes have the potential to collide while crossing the conflict points, indicated by the 24 red circles. For example, CAV $ 2 $ and CAV $ 3 $ have a crossing conflict point.
	\item \emph{Diverging Conflict}: Lane changing and overtaking are not permitted, as explained in Assumption~\ref{asp:LaneChange}. Thus, vehicles on the same lane cannot pass the intersection simultaneously, as indicated by the 14 orange squares. For example, CAV $ 6 $ and CAV $ 7 $ have a diverging conflict point.
	\item \emph{Converging Conflict}: {Vehicles from different lanes cannot enter the same lane simultaneously, as shown by the 6 green arrows.} For example, CAV $ 2 $ and CAV $ 4 $ have a converging conflict point.
\end{enumerate}

{
Apart from the above-mentioned three kinds of conflicts that describe the route conflict of the CAVs, we claim that there exists another conflict type caused by the speed and acceleration limitation of the CAVs. For instance in Fig.~\ref{fig:AllConflicts}, CAV $ 7 $ arrives at the control zone when CAV $ 5 $ is about to reach the stopping line with the designed virtual platoon velocity. If we schedule CAV $ 5 $ and $ 7 $ to pass simultaneously, CAV $ 5 $ has to wait for CAV $ 7 $ near the stopping line for a long time, which jeopardizes traffic safety and efficiency. In this case, CAV $ 7 $ is not supposed to catch up with CAV $ 5 $ at the stopping line even if they have no conflict relationship, because of the limitations on the speed and acceleration of CAV $ 7 $. Hence, we further introduce the fourth type of CAV conflict, namely, the \emph{reachability conflict}.


\begin{enumerate}[resume*]
	\item \emph{Reachability Conflict}: Even if the CAVs have no route conflict, their acceleration and velocity constraints still limit them from passing the intersection simultaneously. For example, as described before, CAV $ 5 $ and CAV $ 7 $ have a reachability conflict.
\end{enumerate}

Assume CAV $ 7 $ reaches the control zone at distance of $ L_\mathrm{ctrl} $ with initial speed of designed virtual platoon velocity $ v_{\mathrm{p}} $. From equation~\eqref{equ:StateSpaceEquation} and~\eqref{equ:Constraints}, if CAV $ 7 $ immediately accelerates with maximum acceleration $ u_{\max} $ until it reaches maximum speed $ v_{\max} $, it still takes $ t_{\min} $ time until it reaches the stopping line as shown in~\eqref{equ:TimeMin}.
\begin{equation}
	\label{equ:TimeMin}
	t_{min} = \frac{v_{\max} - v_\mathrm{p}}{u_{\max}} + \frac{1}{v_{\max}}\left(L_\mathrm{ctrl} - \frac{v_{\max}^{2} - v_\mathrm{p}^{2}}{2u_{\max}}\right).
\end{equation}

We assume CAV $ 5 $ is also running with designed virtual platoon velocity $ v_\mathrm{p} $  within the control zone, and we expect CAV $ 5 $ to maintain its speed in order to obtain a steady traffic state near the stopping line. If CAV $ 5 $ in the control zone is so close to the stopping line that it arrives at the stopping line in less than $ t_{\min} $ time, CAV $ 7 $ is impossible to catch up with CAV $ 5 $. By simplifying~\eqref{equ:TimeMin}, we obtain the judging condition as
\begin{equation}
	\label{equ:Reachability}
	\frac{L_\mathrm{prec}}{v_\mathrm{p}} < \frac{L_\mathrm{ctrl}}{v_{\max}} + \frac{(v_{\max}-v_\mathrm{p})^2}{2u_{\max}v_{\max}},
\end{equation}
where $ L_\mathrm{prec} $ is the distance of the preceding CAV from the stopping line. In other words, preceding vehicles with steady $ v_\mathrm{p} $ form a steady scheduling plan near the intersection. If the new CAV cannot catch up with some of them because of velocity/acceleration constraints, reachability conflict is used to preserve the consistency of scheduling results. Note that most of the existing researches assume that the CAVs can reach the stopping line under all circumstances,~\ie, the reachability conflict is ignored.
}


{
We define different conflict sets to describe the conflict relationship of the CAVs. For each CAV $ i \left(i \leq N, i \in \mathbb{N}^+\right) $, the crossing set is defined as $ \mathcal{C}_{i} $, diverging set as $ \mathcal{D}_{i} $, converging set as $ \mathcal{V}_{i} $, and reachability set as $ \mathcal{R}_{i} $. Note that conflict sets are determined when the CAV reaches the control zone border. At this time, CAV $ i $ is at the control zone border while other vehicles are in the control zone. Hence, other CAVs' indexes $ j $ in CAV $ i $'s conflict sets are smaller than $ i $ itself,~\ie, the elements in CAV $ i $'s conflict sets satisfy~\eqref{equ:ElementRelationship}.
\begin{equation}
	\label{equ:ElementRelationship}
	j < i, \, \mathrm{if}\ j \in \mathcal{C}_{i} \cup \mathcal{D}_{i} \cup \mathcal{V}_{i} \cup \mathcal{R}_{i}.
\end{equation}
}

\begin{table}[!t]
	\renewcommand{\arraystretch}{1.3}
	\centering
	\caption{Conflict Analysis of Example~\ref{exp:1}.}
	\label{tab:ConflictAnalysis}
	\begin{tabular}{c|c|c|c|c|c|c|c}
		\hline
		$ \mathbf{i} $ & \bfseries 1 & \bfseries 2 & \bfseries 3 & \bfseries 4 & \bfseries 5 & \bfseries 6 & \bfseries 7  \\
		\hline
		$ \mathcal{C}_{i} $ & $ \varnothing $ & $ \varnothing $ & $ \{2\} $ & $ \varnothing $ & $ \{2,3\} $ & $ \varnothing $ & $ \varnothing $ \\
		
		$ \mathcal{D}_{i} $ & $ \{0\} $ & $ \{0\} $ & $ \{0\} $ & $ \{0\} $ & $ \{0\} $ & $ \{0\} $ & $ \{6\} $ \\
		
		$ \mathcal{V}_{i} $ & $ \varnothing $ & $ \varnothing $ & $ \varnothing $ & $ \{2\} $ & $ \varnothing $ & $ \{3\} $ & $ \{3\}  $ \\
		
		$ \mathcal{R}_{i} $ & $ \varnothing $ & $ \varnothing $  & $ \varnothing $ & $ \varnothing $ & $ \varnothing $ & $ \varnothing $ & $ \{1,5\} $ \\
		\hline
	\end{tabular}
\end{table}

\begin{example}[Conflict analysis for $ 7 $ CAVs]
	\label{exp:1}
	Consider the example shown in Fig.~\ref{fig:AllConflicts}. CAVs $ 1 $ and $ 2 $ are approaching from the east direction and have two separate destinations. Two CAVs are traveling straight from the west with $ 3 $ on one lane and $ 4 $ on another. CAV $ 5 $ is approaching from the south on the left lane and CAV $ 6,7 $ is on the right lane. Note that CAV $ 7 $ is too far from CAVs $ 1 $ and $ 5 $ to catch up with them at the stopping line.
\end{example}




{
Note that a virtual leading vehicle $ 0 $ is set with a constant velocity $ v_\mathrm{p} $ in front of the CAVs that are closest to the intersection. $ v_\mathrm{p} $ represents pre-defined desired platoon speed in virtual platoon coordination. Without the virtual leading vehicle $ 0 $, CAVs which are closest to the intersection will accelerate to $ v_{\max} $ and make it impossible to form a virtual platoon. With this pre-defined $ v_\mathrm{p} $ and the spanning tree generated in Section~\ref{Sec:iDFST}, the virtual platoon is controlled the same way as typical vehicle platooning. Accordingly, virtual leading vehicle $ 0 $ is added to the diverging set of the CAV closest to the intersection on each lane; therefore, in Example~\ref{exp:1}, $ \mathcal{D}_{k} = \{0\}, k = \{1,2,3,4,5,6\} $. The remaining conflict sets are shown in Table~\ref{tab:ConflictAnalysis}.
}

	\subsection{Improved DFST Method}
	\label{Sec:iDFST}
	
	\begin{figure}[!t]
		\centering
		\subfigure[Conflict Directed Graph\label{fig:ConflictDirectedGraph}]
		{\includegraphics[width=0.46\linewidth]{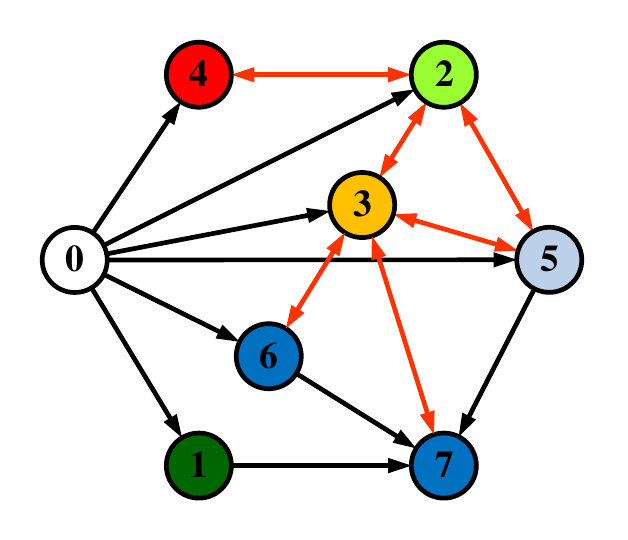}}
		\subfigure[Coexisting Undirected Graph\label{fig:CoexistUndirectedGraph}]
		{\includegraphics[width=0.46\linewidth]{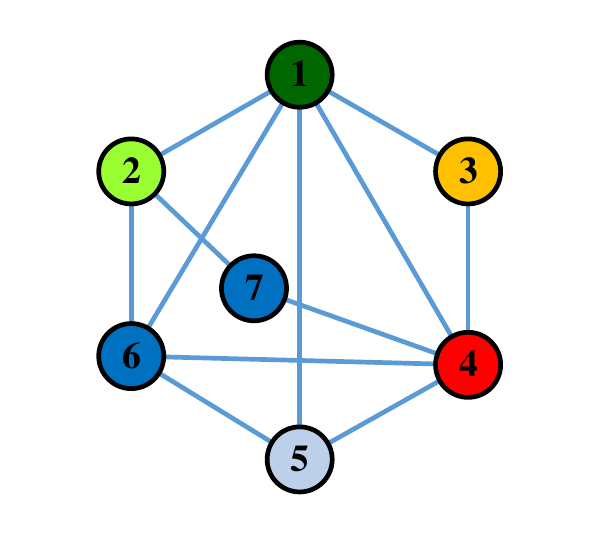}}
		\caption{Fig.~\ref{fig:ConflictDirectedGraph} is the conflict directed graph (CDG). The black unidirectional edges represent the diverging conflicts and reachability conflicts, whereas the red bidirectional edges represent the crossing conflicts and converging conflicts. Fig~\ref{fig:CoexistUndirectedGraph} is the Coexisting Undirected Graph (CUG), which is the complement graph of the CDG (exclude node $ 0 $), and describes the coexistence relationship of the vehicles.}
		\label{fig:ConflictDirectedScenario}
	\end{figure}
	
	Based on the conflict set analysis in Section~\ref{Sec:ConflictAnalysis}, we further define the conflict directed graph (CDG) $ \mathcal{G}_{N+1} $ to represent the conflict relationship between the CAVs.
	
	\begin{definition}[Conflict Directed Graph]
		\label{def:CDG}
		The CDG is denoted as $ \mathcal{G}_{N+1} = \left( \mathcal{V}_{N+1},\mathcal{E}_{N+1} \right) $. If there are $ N $ CAVs in the control zone, we have the node set $ \mathcal{V}_{N+1} = \{0,1,2,\dots,N\} $ (node $0$ denotes the virtual leading vehicle). The unidirectional edge set is defined as $ \mathcal{E}_{N+1}^{u}=\{(i,j) \mid i \in \mathcal{D}_{j} \cup \mathcal{R}_{j}\} $, and bidirectional edge set is defined as $ \mathcal{E}_{N+1}^{b}=\{(i,j) \mid i \in \mathcal{C}_{j} \cup \mathcal{V}_{j}\} $. The edge set is the union of these two sets as $ \mathcal{E}_{N+1} = \mathcal{E}_{N+1}^{u} \cup \mathcal{E}_{N+1}^{b} $.
	\end{definition}
	
	The CDG of the case scenario in Fig.~\ref{fig:AllConflicts} is drawn in Fig.~\ref{fig:ConflictDirectedGraph}. {The nodes in the CDG represent the $ N+1 $ CAVs in the control zone, while the edges represent their conflicts.} The black unidirectional edges represent the diverging conflicts and reachability conflicts. The existence of a unidirectional edge $ (i,j) $ implies that CAV $ j $ is not allowed to overtake CAV $ i $ because of Assumption~\ref{asp:LaneChange} or CAV $ j $ is not able to catch up to CAV $ i $ because it satisfies~\eqref{equ:Reachability}. Thus, CAV $ j $ cannot reach the intersection earlier than CAV $ i $. The red bidirectional edges denote the crossing conflict and converging conflict, which means that the arrival sequence of the CAVs $ i $ and $ j $'s can be exchanged. Note that in the DFST method, the crossing conflict $ \mathcal{C}_i $, diverging conflict $ \mathcal{D}_i $, and converging conflict $ \mathcal{V}_i $ are treated as one union unidirectional-edge conflict. {In Algorithm~\ref{algo:DFST:2}, we will interpret the separation of these conflict sets to improve the performance of the original DFST algorithm.}
	
	It is straightforward that the CDG describes all the conflict relationships of the CAVs. Previous researches have proved that the CDG in the DFST method possesses a DFST, as shown in Lemma~\ref{lemma:CDG}. Because this conclusion is drawn with only unidirectional edges in the CDG, it can be applied to our iDFST method.
	
	\begin{lemma}[\cite{xu2018distributed}]
		\label{lemma:CDG}
		A CDG has a DFST with root node $ 0 $,~\ie, the virtual leading vehicle.
	\end{lemma}
	
	Because the CDG describes all the conflict relationships of the CAVs, and the edges specifically describe the conflict type, a feasible passing order solution can be obtained by building a spanning tree from the CDG. Similar to the general graph theory, the depth of each node in the spanning tree is calculated based on its distance to the root node $ 0 $. Specifically, the depth of the nodes represents the passing order of the CAVs,~\ie, the CAVs at the same depth shall pass the intersection simultaneously. \cite{xu2018distributed} has proved that if the spanning tree is built by the DFST method, the CAVs of the same depth have a conflict-free attribute,~\ie, a feasible spanning tree, which has been proved in Lemma~\ref{lemma:ConflictFree}.
	
	\begin{lemma}[\cite{xu2018distributed}]
		\label{lemma:ConflictFree}
		Consider a virtual platoon characterized by the spanning tree $ \mathcal{G}_{N+1}' $, which is built from the CDG $ \mathcal{G}_{N+1} $. The trajectories of two CAVs at the same depth in $ \mathcal{G}_{N+1}' $ have no conflict relationship with each other.
	\end{lemma}

	{
	Hence, it implies that for a group of CAVs in the control zone, a passing order solution is related to the DFST. It can be also inferred that the depth of the spanning tree is related to the traffic mobility at the intersection. Note that the spanning tree dynamically changes as CAV enters the control zone or arrives at the stopping line, which will be shown in the case study in Section~\ref{Sec:CaseStudy}. 
	
	In our research, CAVs are abstracted as nodes in the spanning tree. Therefore, the largest depth of the spanning tree nodes $ d_\mathrm{all} $ is related to evacuation time as defined in Definition~\ref{def:EvacuationTime}, and the average depth of the spanning tree nodes is related to ATTD as defined in Definition~\ref{def:ATTD}. Note that vehicles of the same depth in the spanning tree are assumed to pass the stopping line simultaneously. It also means that at any time, vehicles in the conflict zone are conflict-free to avoid potential collisions. Therefore, adequate passing time should be given to each layer to ensure that CAV in every trajectory is able to pass the conflict zone. We are aware that if CAVs of different trajectories are treated respectively, traffic efficiency can be further optimized. However, this deviates from our topic of optimality and computation complexity. Thus, the same passing time is assumed for each layer in the spanning tree.
	
	\emph{Evacuation time} represents the overall time cost for all the CAVs to pass the intersection,~\ie, the arrival time of the last CAV. Thus, minimizing the evacuation time is equivalent to finding the smallest possible $ d_\mathrm{all} $ in all the feasible spanning trees. Formally, we describe this problem as
	}

	{
	\begin{proposition}[SMALLEST-DEPTH]
		\label{equ:SmallestDepth}
		$ <\mathcal{G}_{N+1}, \mathcal{G}_{N+1}', k> $: $ \mathcal{G}_{N+1} $ has a feasible spanning tree $ \mathcal{G}_{N+1}' $ with $ d_\mathrm{all} = k $.
	\end{proposition}
	}

	
	
	We propose the iDFST method in Algorithm~\ref{algo:DFST:1} and~\ref{algo:DFST:2} to generate a DFST. For each node $ i $, we first find all its parent nodes in $ \mathcal{G}_{N+1} $ in Lines~\ref{algo:DFST:1:Parent1} and~\ref{algo:DFST:1:Parent2} of Algorithm~\ref{algo:DFST:1}. Note that both the original DFST method and iDFST method are based on the FIFO principle. Even though there are bidirectional edges in the CDG, the passing order is decided according to the arrival sequence of the CAVs. Thus, the larger-index CAVs should not be considered in the passing-order optimization process of the smaller-index CAVs. This also implies that only the smaller-index CAVs are selected as the parent node candidates, corresponding to the conflict definition in~\eqref{equ:ElementRelationship}.

	\begin{algorithm}[tb]
		\caption{Improved Depth-first Spanning Tree Method}
		\label{algo:DFST:1}
		\begin{algorithmic}[1]
			\Require{Conflict Directed Graph $ \mathcal{G}_{N+1} = \left( \mathcal{V}_{N+1},\mathcal{E}_{N+1} \right) $} 
			\Ensure{Improved Depth-first Spanning Tree  $ \mathcal{G}_{N+1}' = \left( \mathcal{V}_{N+1},\mathcal{E}_{N+1}' \right) $}\\
			\textbf{initialize}: Set the depth of node $ 0 $'s layer $d_0=0$
			\For{$i = 1,2,...,N$}
			\State $ P_{u} = \{m \in \mathcal{V}_{N+1} \mid (m,i) \in \mathcal{E}_{N+1}^{u}\} $. Find all parent nodes $ m $ of $ i $ in $ \mathcal{V}_{N+1} $ with unidirectional edges $ (m,i) $
			\label{algo:DFST:1:Parent1}
			\State $ P_{b} = \{n \in \mathcal{V}_{N+1} \mid (n,i) \in \mathcal{E}_{N+1}^{b}\} $. Find all parent nodes $ n $ of $ i $ in $ \mathcal{V}_{N+1} $ with bidirectional edges $ (n,i) $
			\label{algo:DFST:1:Parent2}
			\State $k$ = FIND-OPT-PARENT($ \mathcal{G}_{N+1}' $, $P_{u}$, $P_{b}$) \label{algo:DFST:1:2}
			\State Set node $k$ as the parent node of $i$ in the graph $ \mathcal{G}_{N+1}' $, add a node $i$ and an edge $(k,i)$ to the graph $ \mathcal{G}_{N+1}' $, and set the depth of the node $i$ to $d_{k}+1$
			\EndFor
		\end{algorithmic} 
	\end{algorithm}
	
	\begin{algorithm}[tb]
		\caption{FIND-OPT-PARENT}
		\label{algo:DFST:2} 
		\begin{algorithmic}[1] 
			\Require{$ \mathcal{G}_{N+1}' $, $P_{u}$, $P_{b}$} 
			\Ensure{Optimal Parent node $k$}
			\State  $ l_{u} = \max d_{m}, \mathrm{s.t.}\; m \in P_{u} $
			\State $ L_{b} = \{d_{n} \mid n \in P_{b}\} $
			\State Find $\min{d_k}, \mathrm{s.t.}\; \{ k \in P_{u} \cup P_{b} \mid d_{k}+1 > l_{u}, \mathrm{and}\; (d_{k}+1) \cap L_{b} = \varnothing\}$ \label{algo:DFST:2:1}
			\State \Return $k$
		\end{algorithmic} 
	\end{algorithm}
	
	
	Then, in Line~\ref{algo:DFST:1:2} of Algorithm~\ref{algo:DFST:1},~\ie, Algorithm~\ref{algo:DFST:2}, we treat the conflict types differently. If the parent node $ j $ is in the diverging or reachability conflict set of CAV $ i $, \ie, $ j \in \mathcal{D}_i \cup \mathcal{R}_i $, $ i $ cannot surpass $ j $ because of the overtaking restriction or acceleration limitation. {The target depth of CAV $ i $ should not surpass the depth of the CAVs in $ \mathcal{D}_i $ or $ \mathcal{R}_i $,~\ie, the largest depth $ l_{u} $ of the unidirectional edge parents.} In other cases, the parent node $ j $ is in the crossing conflict set $ \mathcal{C}_i $ or converging set $ \mathcal{V}_i $ of CAV $ i $. It means CAVs $ i $ and $ j $ cannot arrive at the intersection simultaneously, but their arrival order can be exchanged. Either CAV $ i $ or CAV $ j $ can pass the intersection first; therefore, we find the union depth set of the crossing conflict parents $ L_{b} $. To find the optimal parent $ k $, we should select the proper depth of CAV $ i $. Because the depth of the parent $ k $ is $ d_k $, the depth of CAV $ i $ is $ d_{k}+1 $. As mentioned before, $ d_{k}+1 $ should be neither smaller than $ l_{u} $ nor have an intersection with the set $ L_{b} $. Considering traffic efficiency, the depth $ d_{k} $ should be as small as possible, as shown in Line~\ref{algo:DFST:2:1} in Algorithm~\ref{algo:DFST:2}. The application of the iDFST method in Example~\ref{exp:1} is shown in Table~\ref{tab:iDFST} and the output iDFST is plotted in Fig.~\ref{fig:SpanngTree_Optimized}.

	\begin{figure}[!t]
		\centering
		\subfigure[DFST Spanning Tree\label{fig:SpanngTree_Old}]
		{\includegraphics[width=0.6\linewidth]{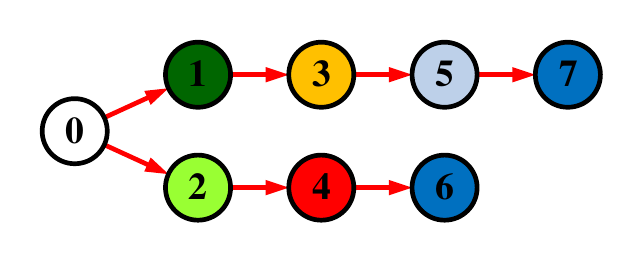}}
		\subfigure[iDFST Spanning Tree\label{fig:SpanngTree_Optimized}]
		{\includegraphics[width=0.52\linewidth]{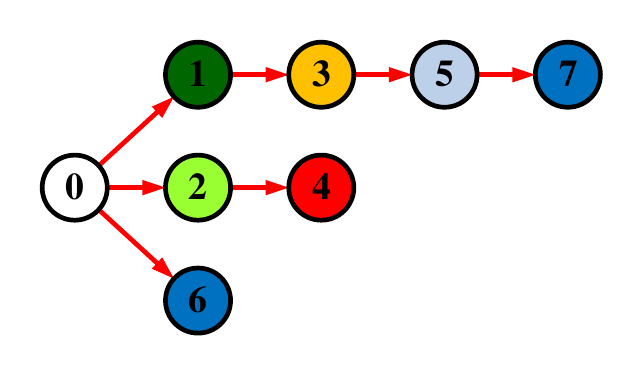}}
		\subfigure[MCC Spanning Tree\label{fig:SpanngTree_MCC}]
		{\includegraphics[width=0.42\linewidth]{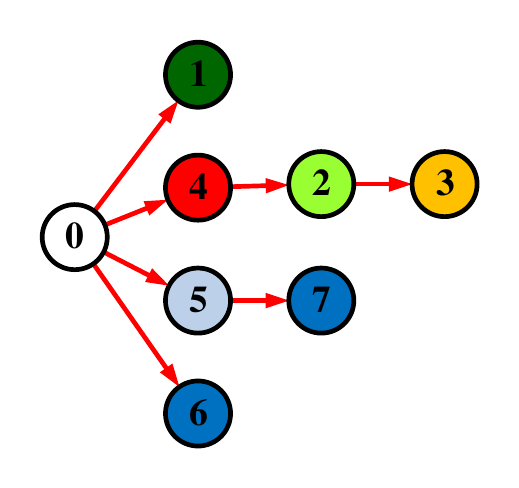}}
		\caption{Comparison of spanning tree results of three different methods. The DFST and iDFST methods yield $ d_\mathrm{all} = 4 $ and the MCC method yields $ d_\mathrm{all} = 3 $. Note that the average depth of the nodes of iDFST result is smaller than that of DFST result.}
		\label{fig:SpanningTree}
	\end{figure}
	
	\begin{table*}[!t]
		\renewcommand{\arraystretch}{1.3}
		\centering
		\caption{Application of the iDFST method to the CDG in Fig.~\ref{fig:ConflictDirectedGraph}, resulting in Fig.~\ref{fig:SpanngTree_Optimized}.}
		\label{tab:iDFST}
		\begin{tabular}{c|c|c|c|c|c|c|c}
			\hline
			Candidate Node $ \mathbf{i} $ & \bfseries 1 & \bfseries 2 & \bfseries 3 & \bfseries 4 & \bfseries 5 & \bfseries 6 & \bfseries 7\\
			\hline
			$ P_{u} = \mathcal{D}_{i} \cup \mathcal{R}_{i} $ & $ \{0\} $ & $ \{0\} $ & $ \{0\} $ & $ \{0\} $ & $ \{0\} $ & $ \{0\} $ & $ \{1,5,6\} $ \\
			$ P_{b} = \mathcal{C}_{i} \cup \mathcal{V}_{i} $ & $ \varnothing $ & $ \varnothing $ & $ \{2\} $ & $ \{2\} $ & $ \{2,3\} $ & $ \{3\} $ & $ \{3\} $ \\
			$ l_{u} $ & $ 0 $ & $ 0 $ & $ 0 $ & $ 0 $ & $ 0 $ & $ 0 $ & $ 3 $ \\
			$ L_{b} $ & $ \varnothing $ & $ \varnothing $  & $ \{1\} $  & $ \{1\} $ & $ \{1,2\} $ & $ \{2\} $ & $ \{2\} $ \\
			\hline
			Selected Parent Node $ \mathbf{k} $ &\bfseries 0  &\bfseries 0 &\bfseries 1 &\bfseries 2 &\bfseries 3 &\bfseries 0 &\bfseries 5 \\
			\hline
			$ d_{i} $ & $ 1 $ &  $ 1 $ & $ 2 $ & $ 2 $ & $ 3 $ & $ 1 $ & $ 4 $ \\
			$ d_\mathrm{all} $ & $ 1 $ &  $ 1 $ & $ 2 $ & $ 2 $ & $ 3 $ & $ 3 $ & $ 4 $ \\
			\hline
		\end{tabular}
	\end{table*}
	
	We promote the proposition~\ref{Pro:DFST2iDFST:EachStep} to show the depth characteristics of the two algorithms.
	\begin{proposition}
	\label{Pro:DFST2iDFST:EachStep}
		In deciding layer for each node, the iDFST method always arranges the node at the same or smaller depth compared to DFST method.
	\end{proposition}


	\begin{IEEEproof}
		We use contradiction to prove the proposition. Without loss of generality, we assume in deciding the layer of vehicle $ i $, the target layer depth of DFST $ d_{i}^{DFST} $ is smaller than that of iDFST $ d_{i}^{iDFST} $ for the first time. Since it is the first time $ d_{i}^{DFST} < d_{i}^{iDFST} $ holds, for vehicles $ j $ with smaller index $ i $, it holds $ d_{j}^{DFST} \geq d_{j}^{iDFST} $. It means that for the vehicles with $ j < i $, iDFST always puts them at smaller depth. Thus for every node $ j $ in $ P_{b} $ and $ P_{u} $, $ d_{j}^{DFST} \geq d_{j}^{iDFST} $. Thus, $ d_{j,\max}^{DFST} \geq d_{j,\max}^{iDFST} $. Since DSFT method use the largest depth of the nodes in $ P_{b} \cup P_{u} $ to decide $ d_{i}^{DFST} $, we obtain $ d_{i}^{DFST} \geq d_{j,\max}^{DFST}$. And if $ d_{i}^{DFST} < d_{i}^{iDFST} $, for vehicle $ i $ we have $ d_{i}^{iDFST} > d_{i}^{DFST} \geq d_{j,\max}^{DFST} \geq d_{j,\max}^{iDFST}$, which leads to $ d_{i}^{iDFST} > d_{j,\max}^{iDFST}$. {It means that in iDFST, the depth of CAV $ i $ $ d_{i}^{iDFST} $ is always bigger than the previous largest depth $ d_{j,\max}^{iDFST} $. Since the spanning tree depths are consecutive integers, it infers that the depth in iDFST is always growing,~\ie, at least $ d_{i}^{iDFST} = d_{i-1}^{iDFST} + 1$. However, in Algorithm~\ref{algo:DFST:2}, depth of CAV $ i $ $ d_{i}^{iDFST} \leq d_{i-1}^{iDFST} $ is possible, which leads to a contradiction.} Therefore, it holds $ d_{i}^{DFST} \geq d_{i}^{iDFST} $ for every node $ i $.
	\end{IEEEproof}
	
	{
	Proposition~\ref{Pro:DFST2iDFST:EachStep} implies that both DFST and iDSFT are based on FIFO principle, and iDSFT further minimizes the depth of every CAV at each step. It also means that iDFST tries to minimize ATTD of the CAV in scheduling it. However, since it is locally optimized, minimum ATTD for all CAVs can not be guaranteed. From proposition~\ref{Pro:DFST2iDFST:EachStep}, since $ d_{i}^{DFST} \geq d_{i}^{iDFST} $ holds for every node, it is evident to conclude theorem~\ref{def:DFST2iDFST}.
	}

	\begin{theorem}
	\label{def:DFST2iDFST}
		In solving the SMALLEST-DEPTH problem as shown in Proposition~\eqref{equ:SmallestDepth}, iDFST method always obtains the same or smaller overall depth of the spanning tree compared to DFST method.
	\end{theorem}

	The main improvement of the iDFST method is in Algorithm~\ref{algo:DFST:2}. In \cite{xu2018distributed}, the conflict parent nodes are treated as one union conflict set $ P = P_{u} \cup P_{b} $. Therefore, the largest depth $ d_{k} $ of all the parent nodes $ P $ are found in the same way. For instance, in the DFST method, the parent nodes of CAV $ 6 $ are the nodes $ \{0,3\} $, which causes the largest depth $ d_{k} $ of the parent nodes to be $ d_{3} = 2 $; therefore, the DFST method has $ d_{6} = 2+1 = 3 $, as shown in Fig.~\ref{fig:SpanngTree_Old}. On the contrary, in the iDFST method, the conflict nodes are separated into two sets $ P_{u}, P_{b} $. By distinguishing whether the parent nodes can be surpassed or not, the smaller $ d_\mathrm{all} $,~\ie, the locally optimal solution is obtained. As shown in Fig.~\ref{fig:SpanngTree_Optimized}, CAV $ 6 $ is ranked $ d_{6} = 1 $ in the iDFST method. Therefore, the depth of CAV $ 6 $ is decreased from $ 3 $ to $ 1 $ by iDFST method. Even if the total depth of DFST and iDFST spanning tree are both $ d_\mathrm{all} = 4 $, the average depth of the nodes in iDFST spanning tree is $ 2 $, which is smaller than that of DFST spanning tree $ 2.28 $.
	{
	\begin{remark}
		We conclude that the DFST method focuses on the feasibility of the passing order problem. The proposed iDFST method additionally considers the optimality,~\ie, finds the smallest depth for each CAV, but the solution is still found CAV by CAV,~\ie, based on the FIFO principle. In other words, the iDFST solution is locally optimal in arranging each CAV. Because each CAV needs to check all the conflict relationships of the parent nodes, the overall computational complexity is $ O(N^2) $. In the simulation, however, the spanning tree of the previous step is passed on to generate the spanning tree of the current step. It means that only one CAV needs to be handled in each calculation step. Therefore the simulation computational complexity is $ O(N) $. What's more, to further reduce the computational complexity, a higher-efficiency data structure in saving the previous scheduling result is needed.
	\end{remark}
	}
	
	\subsection{Minimum Clique Covering Method}
	\label{Sec:MCC}
	
	As mentioned before, we aim to minimize the overall depth $ d_\mathrm{all} $ of the spanning tree, which corresponds to the evacuation time of all the CAVs. If we reconsider the traffic scenario in Fig.~\ref{fig:Intersection}, a maximum number of six CAVs can pass the intersection simultaneously. In other words, because there is a limited number of CAVs in the control zone, the overall depth of the spanning tree can be minimized by maximizing the CAV groups that can drive through the intersection at the same time. From this point of view, another method to describe the conflict relationships of the CAV is to describe their coexistence relationships.
	
	\begin{definition}[Coexisting Undirected Graph]
		\label{def:CUG}
		The CUG is defined as the complement graph of the CDG $ \mathcal{G}_{N+1} $, excluding node $ 0 $. Thus, $ \overline{\mathcal{G}}_{N} = \left(\overline{\mathcal{V}}_{N},\overline{\mathcal{E}}_{N}\right) $, where $ \overline{\mathcal{V}}_{N} = \mathcal{V}_{N+1} - \{0\} $, $ \overline{\mathcal{E}}_{N} = \{(i,j) \mid i,j \in \overline{\mathcal{V}}_{N} , i \neq j, \mathrm{and} \ (i,j) \notin \mathcal{E}_{N+1}\} $.
	\end{definition}

	In Example~\ref{exp:1}, the CDG is drawn in Fig.~\ref{fig:ConflictDirectedGraph} and the CUG in Fig.~\ref{fig:CoexistUndirectedGraph}. Because the CDG edge $ \mathcal{E}_{N+1} $ implies that two CAVs have conflicts and the CUG is the complement graph of the CDG, the CUG edge $ \overline{\mathcal{V}}_{N} $ implies that the two CAVs are conflict-free,~\ie, they can pass through the intersection simultaneously.
	
	Recall that we aim to minimize the overall depth $ d_\mathrm{all} $ of the spanning tree. Because the total number of CAVs,~\ie, the node number $ \left| \overline{\mathcal{V}}_{N} \right|$, is a constant value $ N $, minimizing the overall depth of the spanning tree is equivalent to widening its average width. Thus, solving the optimal passing solution is equivalent to finding the minimum number of groups of the combinations of coexisting CAVs in the CUG. Note that this conclusion corresponds to common sense. The maximum groups of the coexisting CAVs represent the maximizing of the time utility of the intersection, which minimizes the overall evacuation time.
	
	In graph theory, the clique is suitable for describing the coexistence relationship of the CAVs. The definition of the clique is shown in Definition~\ref{def:clique}, and the cliques of sizes $ 3,4,5,6 $ are plotted in Fig.~\ref{fig:Cliques}.
	
	\begin{figure}[!t]
		\centering
		{\includegraphics[width=\linewidth]{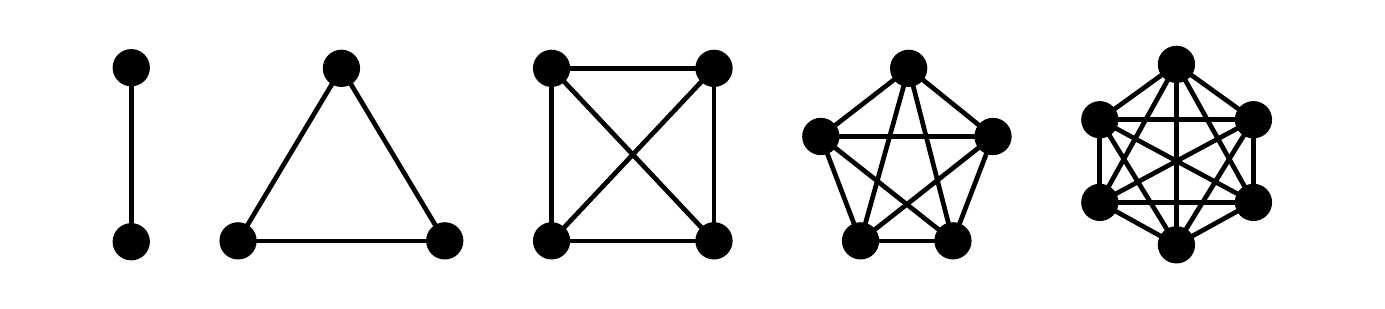}}
		\caption{Cliques of sizes $ 2,3,4,5,6 $.}
		\label{fig:Cliques}
	\end{figure}
	
	\begin{definition}[Clique\cite{luce1949method}]
		\label{def:clique}
		A clique $ C $ in an undirected graph $ G = (V, E) $ is a subset of the nodes, $ C \subseteq V, $ such that every two distinct nodes are adjacent. This is equivalent to the condition that the subgraph of $ G $ induced by $ C $ is a complete graph.
	\end{definition}
	
	The clique $ C $ is the aforementioned group. Because the edges are selected from the CUG $ \overline{\mathcal{G}}_{N} $, the CAVs in one clique can pass through the intersection simultaneously in a conflict-free manner. In Fig.~\ref{fig:CoexistUndirectedGraph}, $ C=\{1,2\} $ is a clique, but it is not the maximum clique because cliques of larger sizes exist,~\eg, $ \{1,2,6\} $. The maximum clique of a graph $ G $ is a clique with the maximum number of nodes. Particularly, the number of nodes in the maximum clique in $ G $ is called the clique number, denoted as $ \omega(G) $. In Fig.~\ref{fig:CoexistUndirectedGraph}, $ \omega(\overline{\mathcal{G}}_{N}) = 4 $, which corresponds to the traffic scenario in Example~\ref{exp:1}, where a maximum number of $ 4 $ CAVs can pass through the intersection simultaneously. Note that in the overall intersection scenario Fig.~\ref{fig:Intersection}, $ \omega(\overline{\mathcal{G}}_{N}) = 6 $ as long as there is an adequate number of CAVs on the coexisting lanes.
	
	Recall that our target is to find the minimum number of groups in the CUG,~\ie, the minimum number of cliques covering all the nodes in the CUG $ \overline{\mathcal{G}}_{N} $. Therefore, we define the MCC problem as follows.
	
	\begin{definition}[Minimum Clique Cover (MCC)~\cite{karp1972reducibility}]
		\label{def:MCC}
		A clique cover of a graph $ G=(V,E) $ is a partition of $ V $ into $ k $ disjoint subsets $ V_{1},V_{2},\dots, V_{k}$, such that for $ 1 \leq i \leq k $, the subgraph induced by $ V_{i} $ is a clique,~\ie, a complete graph. The MCC number of $ G $ is the minimum number of subsets in a clique cover of $ G $, denoted as $ \theta(G) $.
	\end{definition}
	
	The MCC number $ \theta(\overline{\mathcal{G}}_{N}) $ of the CUG represents the minimum number of cliques covering the CUG. Because the cliques in the CUG represent the CAVs that can pass through the intersection simultaneously, these CAVs in the same clique can be scheduled to drive through the intersection at the same time,~\ie, be scheduled at the same depth of the spanning tree. Note that the MCC does not lead to the maximum clique and vice versa. For example, considering the CUG in Fig.~\ref{fig:CoexistUndirectedGraph}, the MCC number $ \theta(\overline{\mathcal{G}}_{N}) = 3 $ and the corresponding cliques are listed in Table~\ref{tab:MCC}. The maximum cliques $ \{1,4,5,6\}$ appears in solution $ 1 $ but not in the subsequent solutions. We conclude that for an arbitrary intersection scenario, the coexistence relationship of the incoming CAVs is perfectly depicted in the CUG $ \overline{\mathcal{G}}_{N} $. The clique number $ \omega(\overline{\mathcal{G}}_{N}) $ of the CUG represents the maximum number of CAVs that can simultaneously pass through the intersection, whereas the MCC number $ \theta(\overline{\mathcal{G}}_{N}) $ represents the possible minimum passing order solution.
	
	\begin{table}[!t]
		\renewcommand{\arraystretch}{1.3}
		\centering
		\caption{Possible minimum clique cover solutions of Fig.~\ref{fig:CoexistUndirectedGraph}. Note that $ \theta(\overline{\mathcal{G}}_{N}) = 3$ in this graph.}
		\label{tab:MCC}
		\begin{tabular}{c|c|c|c}
			\hline
			\diagbox[width=8em]{\textbf{Solutions}}{\textbf{Subsets}} & $ {V}_{1} $ & $ {V}_{2} $  & $ {V}_{3} $ \\
			\hline
			\bfseries 1 & $ \{1,4,5,6\} $ & $ \{2,7\} $ & $ \{3\} $\\
			\bfseries 2 & $ \{4,5,6\} $ & $ \{2,7\} $ & $ \{1,3\} $\\
			\bfseries 3 & $ \{1,3,4\} $ & $ \{5,6\} $ & $ \{2,7\} $\\
			\bfseries 4 & $ \{1,5,6\} $ & $ \{2,7\} $ & $ \{3,4\} $\\
			\hline
		\end{tabular}
	\end{table}
	
		%
	
	Note that $ \theta(\overline{\mathcal{G}}_{N}) = 3 $ in CUG Fig.~\ref{fig:CoexistUndirectedGraph}, which means the spanning tree generated by the MCC solutions are in $ d_\mathrm{all} = 3 $. It also means that the theoretical evacuation times of these solutions are the same,~\ie, the theoretical values of $ t_\mathrm{evac} $ are the same. Thus, the performance index of $ t_\mathrm{evac} $ in~\eqref{equ:EvacuationTime} corresponds to $ \theta(\overline{\mathcal{G}}_{N}) $, which is the global optimal passing order considering the evacuation time. {In addition, we further consider the ATTD among these solutions as a secondary index. The definition of $ t_\mathrm{ATTD} $ in~\eqref{equ:TravelTimeDelay} can be rewritten in graphical terms as
	\begin{equation}
		\label{equ:ATTD_MCC}
		\begin{aligned}
			\min \; & \sum_{i=1}^{k} d_{i}|V_{i}|, \, i \in \mathbb{N}^+\\
			\mathrm{subject\; to: }\; & \sum_{i=1}^{k} |V_{i}| = N,
		\end{aligned}
	\end{equation}
	where $ V_i, \, 1 \leq i \leq k $ are the $ k $ cliques obtained from MCC algorithm. Since clique $ V_i $ is a sub-graph induced from CUG $ \overline{\mathcal{G}}_{N} $, $ |V_{i}| $ represents the node number of clique $ V_i $. Because vehicles of number $ |V_{i}| $ in one clique $ V_{i} $ are arranged in the same layer,~\ie, of the same depth $ d_{i} $, $ \sum_{i=1}^{k} d_{i}|V_{i}| $ represents the weighted summing depth of the spanning tree. Since the total vehicle number is a constant value ($ \sum_{i=1}^{k} |V_{i}| = N $), minimizing the average depth of the spanning tree equals to $ \min \; \sum_{i=1}^{k} d_{i}|V_{i}|, \, i \in \mathbb{N}^+ $.
	}
	
	In this circumstance, it is evident that the subsets $ V_{i} $ should be arranged in the descending order to decrease the average $ d_{i} $ of $ N $ CAVs in~\eqref{equ:ATTD_MCC}. Because of the same reason, we prefer to choose the solutions with the maximum cliques when they have the same $ \theta(\overline{\mathcal{G}}_{N}) $. For example, in Table~\ref{tab:MCC}, we tend to choose the Solution $ 1 $, and the corresponding optimized spanning tree is scheduled as $ \{1,4,5,6\} \rightarrow \{2,7\} \rightarrow \{3\}$, as shown in Fig.~\ref{fig:SpanngTree_MCC}.
	
%
%
%
	
	\begin{algorithm}[tb]
		\caption{Minimum Clique Cover Method}
		\label{algo:MCC} 
		\begin{algorithmic}[1] 
			\Require{Coexisting Undirected Graph $\overline{\mathcal{G}}_{N} = \left(\overline{\mathcal{V}}_{N},\overline{\mathcal{E}}_{N}\right) $} 
			\Ensure{Spanning Tree $ \mathcal{G}_{N+1}' = \left( \mathcal{V}_{N+1},\mathcal{E}_{N+1}' \right) $}
			\State Calculate the complement graph $\overline{\overline{\mathcal{G}}}_{N} = \mathcal{G}_{N}$ \label{algo:MCCBegin}
			\State Find the breadth-first search sequence $ K=(v_{1},v_{2},\dots,v_{N}) $ of the nodes in $ \mathcal{G}_{N} $
			\For{each node $ v_{i} $ of $ \mathcal{G}_{N} $ in the sequence $ K $} 
			\State assign node $ v_{i} $ the smallest possible clique index
			\EndFor \label{algo:MCCEnd}
			\State Rank $ V_{1}, V_{2}, \dots,V_{k}$ in the descending order and obtain the spanning $ \mathcal{G}_{N+1}' $ \label{algo:MCC:Spanning} 
			\State Exchange the conflicting CAVs of $ \mathcal{G}_{N+1}' $ in the same lane if necessary
			\State \Return $ \mathcal{G}_{N+1}' $
		\end{algorithmic} 
	\end{algorithm}
	
	After introducing the MCC problem, we describe the formal complexity of the original SMALLEST-DEPTH problem in Theorem~\ref{def:NPC}.
	\begin{theorem}
		\label{def:NPC}
		SMALLEST-DEPTH problem is NP-complete.
	\end{theorem}
	
	\begin{IEEEproof}
		First, we prove that $ \text{SMALLEST-DEPTH} \in \text{NP} $. Assume we have a CDG $ \mathcal{G}_{N+1} = \left( \mathcal{V}_{N+1},\mathcal{E}_{N+1} \right) $, a spanning tree $ \mathcal{G}_{N+1}' = \left( \mathcal{V}_{N+1},\mathcal{E}_{N+1}' \right) $, and an integral $ k $. The certificate we choose is $ k = |\mathcal{V}_{N+1}| $. The verification algorithm consists of two procedures,~\ie, whether the spanning tree $ \mathcal{G}_{N+1}' $ is feasible and $ d_\mathrm{all} $ of the spanning tree is $ k $. Checking the edge connections of $ \mathcal{E}_{N+1} $ and the depth of each node of $ \mathcal{V}_{N+1} $ can be evidently accomplished in polynomial time.
		
		Secondly, we prove that the SMALLEST-DEPTH problem defined in Proposition~\eqref{equ:SmallestDepth} is NP-hard by showing that it can be reduced to the MINIMUM-CLIQUE-COVER problem~\cite{karp1972reducibility} defined in Definition~\ref{def:MCC},~\ie, SMALLEST-DEPTH $ \leq_{p} $ MINIMUM-CLIQUE-COVER. As mentioned before, given a CDG $ \mathcal{G}_{N+1} = \left( \mathcal{V}_{N+1},\mathcal{E}_{N+1} \right) $, we define the corresponding CUG as $ \overline{\mathcal{G}}_{N} = \left(\overline{\mathcal{V}}_{N},\overline{\mathcal{E}}_{N}\right) $, where $ \overline{\mathcal{V}}_{N} = \mathcal{V}_{N+1 } - \{0\} $, $ \overline{\mathcal{E}}_{N} = \{(i,j) \mid i,j \in \overline{\mathcal{V}}_{N}, i \neq j, \mathrm{and} \ (i,j) \notin \mathcal{E}_{N+1}\} $. The transformation of the CDG into CUG is easily achieved in polynomial time.
		
		If the CDG $ \mathcal{G}_{N+1} $ has a feasible spanning tree $ \mathcal{G}_{N+1}' $ with $ d_\mathrm{all} = k $, we claim that the corresponding CUG $ \overline{\mathcal{G}}_{N} $ has  $ \overline{\mathcal{V}}_{N} $ partitioned into $ k $ disjoint cliques. Because $ \mathcal{G}_{N+1}' $ is a feasible spanning tree, if $ i,j $ are the nodes of same depth, they should not have a conflicting relationship,~\ie, $ (i,j) \notin \mathcal{E}_{N+1} $. Because the CUG is the complement graph of the CDG, it is evident that $ (i,j) \in \overline{\mathcal{E}}_{N} $. Thus, the nodes of the same depth in the spanning tree  $ \mathcal{G}_{N+1}' $ are connected with each other,~\ie, they form cliques in $ \overline{\mathcal{G}}_{N} $. The overall depth of the spanning tree $ k $ corresponds to the clique cover number of $ \overline{\mathcal{G}}_{N} $.
		
		On the contrary, we assume the the CUG $ \overline{\mathcal{G}}_{N} $ has $ \overline{\mathcal{V}}_{N} $ partitioned into $ k $ disjoint cliques. We know that $ k $ disjoint cliques lead to a feasible passing order, that is, the CDG $ \mathcal{G}_{N+1} $ has a feasible spanning tree $ \mathcal{G}_{N+1}' $ with $ d_\mathrm{all} = k $.
	\end{IEEEproof}
	
	
	We have proven that the SMALLEST-DEPTH problem can be reduced to the MCC problem, which implies that the MCC number $ \theta(\overline{\mathcal{G}}_{N}) $ of the CUG is the smallest possible $ d_\mathrm{all} $ of the spanning tree solutions,~\ie, the global optimal solution. However, solving the NP-hard problem is a difficult task in real deployment. In the simulation, for a small number of vehicles, we applied the brute-force method to find the MCC number,~\ie, the strictly global optimal solution. In the previous work, we have shown that when the traffic scenario is simple,~\ie, a maximum number of two vehicles are allowed to pass the intersection simultaneously, the CAV scheduling problem is reduced to a maximum matching problem~\cite{chen2021graph}. In this case, the global optimal solution is obtained in $ O(n^4) $ time.
	
	For a large number of vehicles and complicated intersection scenarios, the dimensional explosion makes the brute-force method impossible to deploy. Therefore, we apply a practical approach to solve the problem heuristically, as shown in Line~\ref{algo:MCCBegin} to~\ref{algo:MCCEnd} of Algorithm~\ref{algo:MCC}. The MCC problem of $ G $ is proved to be reduced to the graph coloring problem of $ \overline{G} $~\cite{garey1979guide}, and there are numerous heuristic methods to solve the graph coloring problem. Recall that we intend to find the solutions with larger cliques. Thus, we first generate a node sequence $ K=(v_{1},v_{2},\dots,v_{N}) $ by breadth-first search (BFS). Then, we greedily assign node $ v_{i} $ the smallest possible color,~\ie, the clique index according to the node sequence $ K $, forming the subset cliques $ V_{1},\dots,V_{k}$. Line~\ref{algo:MCC:Spanning} is the spanning process, which arranges the cliques into a spanning tree and has a constant calculation time.
	
	{
	\begin{remark}
		The BFS sequence generation needs $ O(\left|{\mathcal{E}_{N}}\right| + \left|{\mathcal{V}_{N}}\right|) $. Afterwards, the MCC is solved greedily in $ O(\left|{\mathcal{V}_{N}}\right|) $ and the ranking of the spanning tree requires a constant time. Hence, the overall computational complexity of the heuristically solved MCC method,~\ie, Algorithm~\ref{algo:MCC} is $ O(\left|{\mathcal{E}_{N}}\right| + 2\left|{\mathcal{V}_{N}}\right| + 1) = O(N) $, which is the same as $ O(N) $ of the DFST and iDFST methods. In the simulation, we will show that even if the heuristically solved MCC method is not the strictly globally optimal solution, it outperforms the other two methods.
	\end{remark}
	}
	
	\subsection{Distributed Control}
	\label{Sec:DistributedControl}
	In the ideal communication condition, every CAV obtains driving information,~\eg, position and velocity, from the other CAVs in real-time. However, the communication resource,~\eg, the bandwidth, is always limited around an intersection. Therefore, we design a hierarchical framework to lower the communication burden. In the upper level, the central coordinator collects the information of the CAVs and assigns the scheduled arrival plans to them. In the lower level, a CAV executes its arrival time through distributed control. Similar to the virtual platoon analysis in~\cite{xu2018distributed}, we design the distributed control as follows.
	
	\subsubsection{Geometric Topology}
	{As mentioned in Section~\ref{Sec:Methodology}, all the three methods,~\ie, DFST, iDFST, and MCC, lead to a feasible spanning tree $ \mathcal{G}_{N+1}' = \left( \mathcal{V}_{N+1},\mathcal{E}_{N+1}' \right) $, as shown in Fig.~\ref{fig:SpanningTree}, which forms the geometric topology of the virtual platoon. Because each node $ \mathcal{V}_{i} $ in $ \mathcal{G}_{N+1}' $ represents a CAV $ i $ and the edge $ \mathcal{E}_{i}' = (j,i) $ represents the virtual preceding CAV $ j $, the CAV can always follow its parent node CAV $ \mathcal{V}_{j} $. In other words, CAV $ i $ and its parent CAV $ j $ intend to maintain a constant desired car-following distance $ D_\mathrm{des} $ and the same velocity $ v_\mathrm{des} $.
	\begin{equation}
		\label{equ:CarFollowingDistance}
		\left\{
		\begin{array}{l}
			\displaystyle\lim_{t \rightarrow \infty}\left\|v_{i}(t)-v_{j}(t)\right\|=0 \\
			\displaystyle\lim_{t \rightarrow \infty}\left(p_{j}(t)-p_{i}(t)-D_\mathrm{des}\right)=0
		\end{array}, i \in \mathbb{N}^+,  \right.
	\end{equation}
	where $ p_{i}(t), p_{j}(t) $ denote the remaining distance of the CAVs $ i $ and $ j $ to the stopping line, and $ v_{i}(t), v_{j}(t) $ for their velocities. Lemma~\ref{lemma:ConflictFree} has proven that the CAVs at the same depth have no conflict relationship. Therefore, if CAVs $ i $ and $ j $ are of the same depth,~\eqref{equ:CarFollowingDistance} also aligns them at the same depth in arriving at the stopping line.
	}
	
	
	
	\subsubsection{Communication Topology}
	
	\begin{figure}[!t]
		\centering
		\subfigure[iDFST\label{fig:CommunicationTopology_iDFST}]
		{\includegraphics[width=0.8\linewidth]{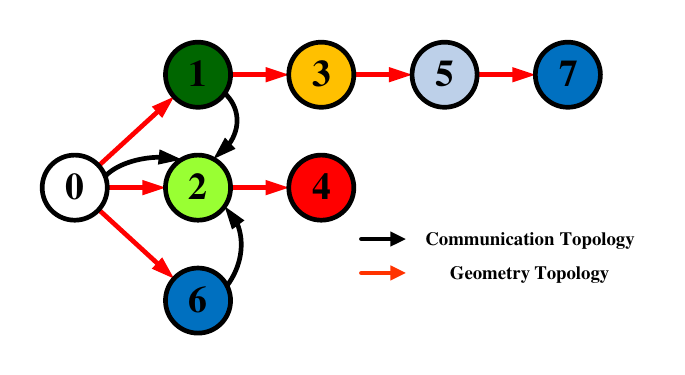}}
		\subfigure[MCC\label{fig:CommunicationTopology_MCC}]
		{\includegraphics[width=0.8\linewidth]{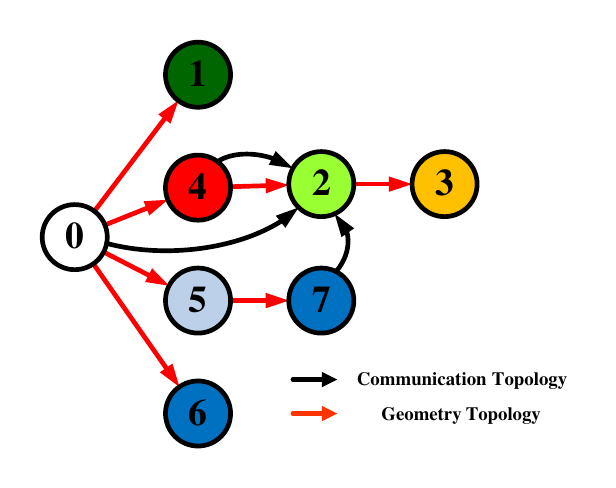}}
		\caption{Predecessor--leader following (PLF) communication topology illustration of the iDFST and MCC methods. Red edges represent the geometric topology of the virtual platoon, which is the result of different scheduling methods, as shown in Fig.~\ref{fig:SpanningTree}. The black edges represent the PLF communication topology of the virtual platoon, where each CAV collects information from the virtual preceding vehicle and the virtual leading vehicle.}
		\label{fig:CommunicationTopology}
	\end{figure}
	
	The communication topology represents the information exchange route among the CAVs. Various communication topologies have been studied in the field of vehicle platoon control~\cite{zheng2015stabilityandscalability}. For simplicity, we use the predecessor--leader following (PLF) topology, where each CAV communicates with its preceding vehicle in the virtual platoon and the virtual leading vehicle. Similar to the topology analysis in~\cite{xu2018distributed, zheng2015stabilitymargin}, we introduce $ \overline{\mathcal{G}}_{N+1}' = \left\{ \overline{\mathcal{V}}_{N+1},\overline{\mathcal{E}}_{N+1} \right\} $ as the communication graph of the spanning tree $ \mathcal{G}_{N+1}' = \left\{ \mathcal{V}_{N+1},\mathcal{E}_{N+1}' \right\} $. The nodes still represent the CAVs; thus, $ \overline{\mathcal{V}}_{N+1} = \mathcal{V}_{N+1} $. Meanwhile, the edges $ \overline{\mathcal{E}}_{N+1} $ represent the information exchange among the CAVs. The communication topologies of the iDFST and MCC spanning trees are shown in Fig.~\ref{fig:CommunicationTopology}. Specifically, the CAVs collect information from the following two types of CAVs.
	
	\begin{enumerate}[1)]
		\item \emph{Virtual leading vehicle}: because the CAV obtains its target layer in the scheduled spanning tree, every CAV in the spanning tree communicates with the virtual leading vehicle $ 0 $.
		\begin{equation}
			\left\{(0, j) \mid j \in \overline{\mathcal{V}}_{N+1}, \, j \neq 0 \right\} \in \overline{\mathcal{E}}_{N+1}.
		\end{equation}
		
		\item \emph{Virtual preceding vehicle}: the second type is the virtual preceding vehicle in the spanning tree $ \mathcal{G}_{N+1}' = \left\{ \mathcal{V}_{N+1},\mathcal{E}_{N+1}' \right\} $. This term is used to help the CAVs maintain a safe distance from the upper layer in the spanning tree. Note that the preceding vehicle in the spanning tree does not symbolize the preceding vehicle of the same lane in the simulation and vice versa. The CAVs $ 0 $ and $ 4 $ are the preceding vehicles of CAV $ 2 $ in the iDFST and MCC spanning trees. respectively.
		\begin{equation}
			\left\{(i, j) \mid (i,j) \in \mathcal{E}_{N+1}',\, i,j \neq 0 \right\} \in \overline{\mathcal{E}}_{N+1}.
		\end{equation}
		
	\end{enumerate}
	
	The adjacency matrix $ \boldsymbol{A}=\left[a_{ij}\right] \in \mathbb{R}^{N \times N} $ and pinning matrix $ \boldsymbol{Q}=\left[q_{ij}\right] \in \mathbb{R}^{N \times N} $ are further introduced to represent the communication topology.
	\begin{equation}
		a_{ij}=\left\{
		\begin{array}{rl}
			1, & \mathrm{if} \, (i,j) \in \overline{\mathcal{E}}_{N+1} \\
			0, & \mathrm{else}
		\end{array}\right. , \, i,j \in \mathbb{N}^{+},
	\end{equation}
	where $ a_{ij} = 1 $ indicates that CAV $ i $ receives information from CAV $ j $. 
	
	\begin{equation}
		q_{ij}=\left\{
		\begin{array}{rl}
			1, & \mathrm{if} \, i=j, (0,j) \in \overline{\mathcal{E}}_{N+1} \\
			0, & \mathrm{else}
		\end{array}\right. , \, i,j \in \mathbb{N}^{+},
	\end{equation}
	where $ q_{ij} = 1 $ indicates that CAV $ i $ receives information from the virtual leading vehicle $ 0 $. {Therefore, $ \boldsymbol{Q} $ is a diagonal matrix.} 
	
	Further, a Laplacian matrix $ \boldsymbol{L}=\left[l_{ij}\right] \in \mathbb{R}^{N \times N} $ is also introduced as
	\begin{equation}
		l_{ij}=\left\{
		\begin{array}{ll}
			-a_{ij}, &  i \neq j \\
			\sum_{k=1}^{N} a_{ik}, & i = j
		\end{array}\right. , \, i,j \in \mathbb{N}^{+}.
	\end{equation}
	
	Note that we assume the information exchange is bidirectional; thus, $ a_{ij} = a_{ji} $. Hence, both $ \boldsymbol{A} $, $ \boldsymbol{Q} $ and $ \boldsymbol{L} $ are symmetric matrices.
	
	\begin{remark}
		The communication topology graph $ \overline{\mathcal{G}}_{N+1}' = \left\{ \overline{\mathcal{V}}_{N+1},\overline{\mathcal{E}}_{N+1} \right\} $ is designed based on the geometric topology graph of the spanning tree $ \mathcal{G}_{N+1}' = \left\{ \mathcal{V}_{N+1},\mathcal{E}_{N+1}' \right\} $. Because the geometric topology represents the feasibility and optimality of an arrival plan, the communication topology represents the communication realization of the arrival plan. \cite{xu2018distributed} does not specify the communication topology. However, in this paper, we use PLF to simplify the communication topology and demonstrate the effectiveness of the algorithm. Topologies with greater complexity will be addressed in future research.
	\end{remark}
	
	
	\subsubsection{Controller Design}
	{The control input of CAV $ i $ is calculated according to the communication topology. The car-following distance in the virtual platooning is defined as in~\eqref{equ:CarFollowingDistance}. Hence, the virtual platoon controller is designed as follows.}
	
	Firstly, a union set $ \mathbb{I}_{i} $ is defined to describe the information exchange of CAV $ i $, as follows:
	\begin{equation}
		\mathbb{I}_{i}=\left\{j \mid a_{i j}=1\right\} \cup \left\{0 \mid q_{i i}=1\right\}.
	\end{equation}

	{
	The distance and velocity errors are defined as
	\begin{equation}
		\begin{array}{l}
			\delta_{\mathrm{p}}^{(i,j)} = p_{j}(t)-p_{i}(t)-D_\mathrm{des}\left(d_{j}-d_{i}\right) \\
			\delta_{\mathrm{v}}^{(i,j)} = v_{i}(t)-v_{j}(t)
		\end{array}, \, j \in \mathbb{I}_{i},
	\end{equation}
	where $ \delta_{\mathrm{p}}^{(i,j)} $ is the car-following distance error of CAV $ i $, and $ \delta_{\mathrm{v}}^{(i,j)} $ is the car-following velocity error considering all the CAVs in $ \mathbb{I}_{i} $.
	
	 Note that the arrival plans are described by the virtual platoon. Therefore, typical platoon controllers~\cite{xu2018distributed, zheng2015stabilitymargin, zheng2015stabilityandscalability} can be used in controller design. Since our research focuses on the scheduling problem rather than the control problem, a linear feedback controller is designed as follows:
	\begin{equation}
		\label{equ:Feedback_1}
		\begin{aligned}
			u_{i} 
			=&-\sum_{j \in \mathbb{I}_{i}} k_{p} \delta_{\mathrm{p}}^{(i,j)}-\sum_{j \in \mathbb{I}_{i}} k_{v}\delta_{\mathrm{v}}^{(i,j)} \\
			=&-k_{p} \sum_{j \in \mathbb{I}_{i}} a_{ij}\left(p_{j}(t)-p_{i}(t)-D_\mathrm{des}\left(d_{j}-d_{i}\right)\right) \\
			&-k_{p}q_{ii}\left(p_{0}(t)-p_{i}(t)-D_\mathrm{des}\left(d_{0}-d_{i}\right)\right) \\
			&-k_{v} \sum_{j \in \mathbb{I}_{i}}a_{ij}\left(v_{i}-v_{j}\right)-k_{v}q_{ii}\left(v_{i}-v_\mathrm{p}\right),
		\end{aligned}
	\end{equation}
	where $ k_{p} $ and $ k_{v} $ are the feedback gains of the distance and velocity errors of CAV $ i $. Because we consider a homogeneous scenario, the same gains are set for all the CAVs.
	}
	
	As mentioned before, we consider a second-order vehicle dynamics as shown in~\eqref{equ:StateSpaceEquation}. We define the car following errors as the new vehicle state.
	\begin{equation}
		\bar{\boldsymbol{x}}_{i}=\left[
		\begin{array}{c}
			\bar{x}_{i, 1} \\
			\bar{x}_{i, 2} \\
		\end{array}\right]=\left[
		\begin{array}{c}
			p_{0}-p_{i}-D_\mathrm{des}\left(d_{0}-d_{i}\right) \\
			v_{i}-v_\mathrm{p} \\
		\end{array}\right], i \in \mathbb{N}^{+}.
	\end{equation}
	The vehicle input remains the same,~\ie, $ \bar{u}_{i} = u_{i} $.
	
	Therefore, the car-following vehicle dynamic model is
	\begin{equation}
		\dot{\bar{\boldsymbol{x}}}_{i}=\boldsymbol{A} \bar{\boldsymbol{x}}_{i}+\boldsymbol{B} \bar{u}_{i}, \, i \in \mathbb{N}^{+}.
	\end{equation}
	
	The linear feedback controller is simplified to
	\begin{equation}
		\bar{u}_{i}=-k_{p} \sum_{j}\left(l_{i j}+q_{i j}\right) \bar{x}_{j, 1}-k_{v} \sum_{j}\left(l_{i j}+q_{i j}\right) \bar{x}_{j, 2},\, j \in \mathbb{I}_{i}.
	\end{equation}
	
	Defining {$ \boldsymbol{k} = \left[k_{p},k_{v}\right]^{T} $}, we have
	\begin{equation}
		\bar{u}_{i}=-\sum_{j} \left(l_{ij}+q_{ij}\right) \boldsymbol{k}^{T}\bar{\boldsymbol{x}}_{i}, \, i \in \mathbb{N}^{+}.
	\end{equation}

\section{Simulation}
\label{Sec:Simulation}
\subsection{Simulation Environment and Performance Index}
\begin{table*}[t]
	\centering
	\caption{Key Parameters.}
	\label{tab:Parameters}
	\begin{tabular}[c]{clcl}
		\toprule
		{\textbf{Types}} & {\textbf{Parameter}} &  {\textbf{Symbol}} & {\textbf{Value}} \\
		\midrule
		\multirow{3}{0.2\textwidth}{\centering\textbf{Simulation Parameters}}  & Simulation step & - &  0.1 s \\
		& Initial vehicle speed & - & {15 m/s} \\
		& Control zone length & $ L_{\mathrm{ctrl}} $ & 900 m \\
		\cline{0-3}
		\multirow{8}{0.2\textwidth}{\centering\textbf{Controller Parameters}}
		& Feedback gain of distance error & $ k_{p} $ & 0.1\\
		& Feedback gain of velocity error & {$ k_{v} $} & 0.3\\
		& Desired velocity of the virtual leading vehicle & $ v_\mathrm{p} $ & 10 m/s\\
		& Desired car-following distance of the virtual platoon & $ D_\mathrm{des} $ & 30 m\\
		& Maximum acceleration & $ u_{\max} $ & 5 m/s$^{2} $ \\
		& Minimum acceleration & $ u_{\min} $ & -6 m/s$^{2} $ \\
		& Maximum velocity  & $ v_{\max} $ & {15 m/s} \\
		& Minimum velocity & $ v_{\min} $ & 0 m/s  \\
		\bottomrule
	\end{tabular}
\end{table*}

The traffic simulation was conducted in SUMO, which is widely used in traffic researches\cite{lopez2018microscopic}. The simulation was run on an Intel Core i7-12700F 2.1 GHz processor.{ The intersection scenario and lane direction settings are the same as shown in Fig.~\ref{fig:Intersection}, and the control zone length $ L_{\mathrm{ctrl}} = 900 m $. Note that CAVs should converge to their target layer when arriving at the stopping line. Therefore, adequate control zone length is provided to enable CAVs to form an equilibrium virtual platoon state. It is an interesting future research topic to analyze the shortest control zone length needed in algorithm deployment.} The CAV arrival is assumed to be a Poisson distributed flow, given by
\begin{equation}
	\label{equ:Poisson}
	P(X=k)=\frac{\lambda^{k}}{k!} e^{-\lambda}, k=0,1,\cdots,
\end{equation}
where $ X $ represents the vehicle's arrival at the control zone. $ \lambda $ is the expected value as well as the variance of the Poisson distribution.

{
Traffic efficiency is the main consideration in the proposed scheduling algorithms. To better show the improvement of the proposed algorithms, we further introduce fuel consumption as a performance indicator to evaluate the fuel economy. For the overall traffic efficiency, the evacuation time is as defined in Definition~\ref{def:EvacuationTime}. In terms of individual benefits, the ATTD is considered as defined in Definition~\ref{def:ATTD}. As for fuel consumption, the HBEFA3 model is used to evaluate the fuel consumption of the CAVs, which is integrated with SUMO~\cite{keller2010handbook}. Other key simulation parameters are shown in Table~\ref{tab:Parameters}. 
}

{
\begin{remark}
	For safety concerns, $ D_\mathrm{des} $ is designed relatively conservatively to avoid vehicles in different layers from collision, and vehicles are assumed to maintain desired virtual platoon speed $ v_\mathrm{p} $ in the intersection conflict zone. Therefore, maximum speed $ v_{\max} $ can only be reached in CAVs' intersection-approaching behavior. These speed settings accord with the local traffic regulations. Furthermore, it would be an interesting future direction to consider vehicle speed trajectory planning in the conflict zone and robust controller, which helps to further increase traffic efficiency and platoon performance.
\end{remark}
}

\subsection{Case Study of Algorithm Comparison}
\label{Sec:CaseStudy}
\begin{figure*}[!t]
	\centering
	\subfigure[Conflict Directed Graph\label{fig:AlgorithmCompare_CDG}]
	{\includegraphics[width=0.32\linewidth]{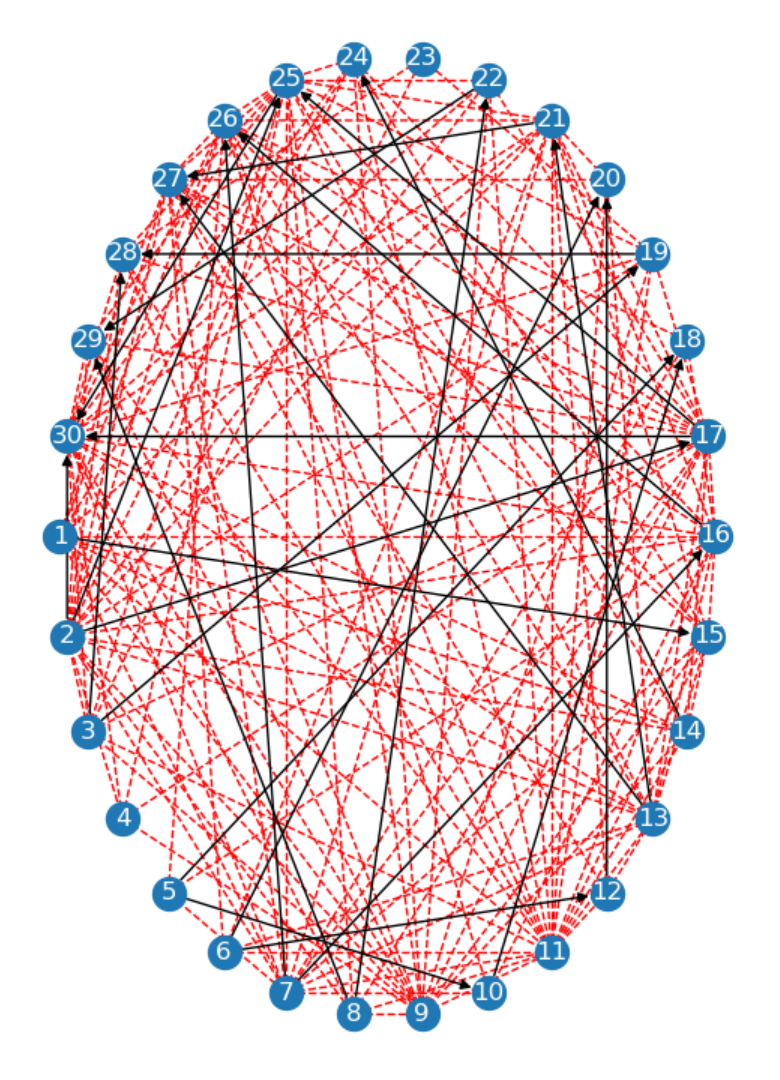}}
	\subfigure[Coexisting Undirected Graph\label{fig:AlgorithmCompare_CUG}]
	{\includegraphics[width=0.32\linewidth]{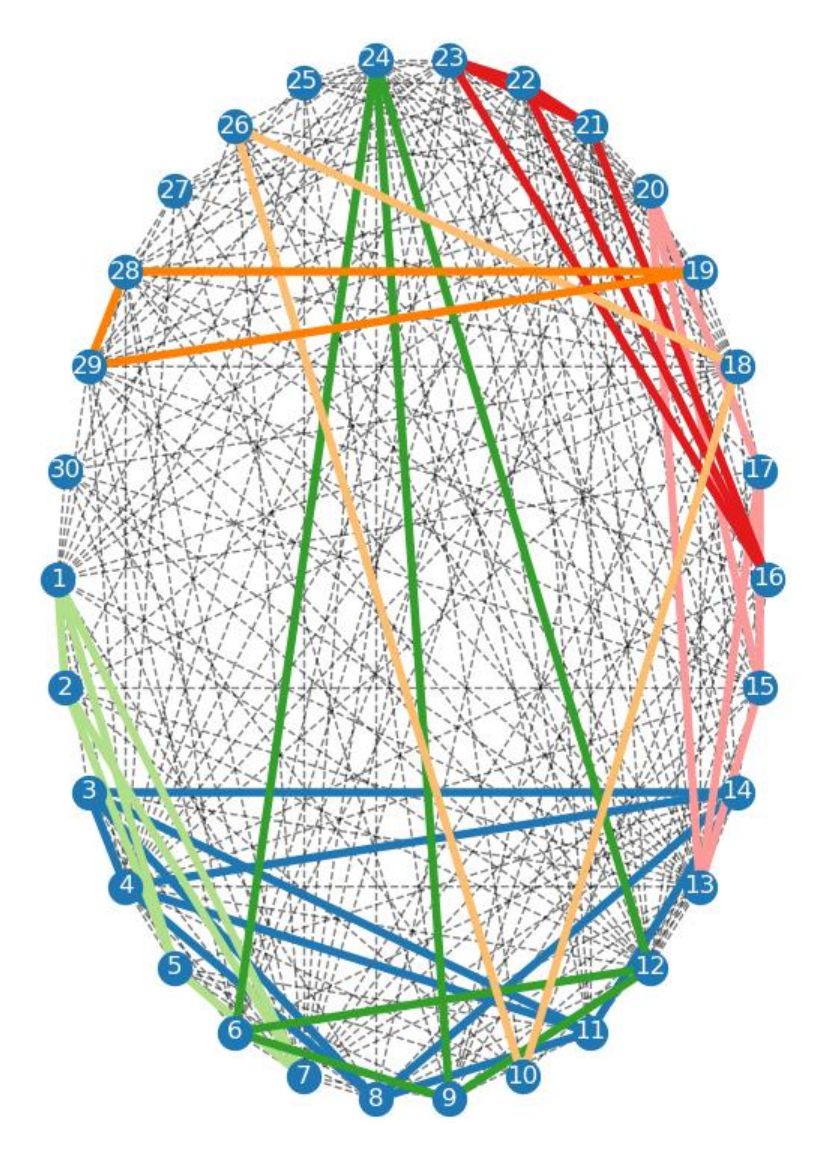}}
	
	\subfigure[DFST Method Spanning Tree\label{fig:AlgorithmCompare_DFST}]
	{\includegraphics[width=0.32\linewidth]{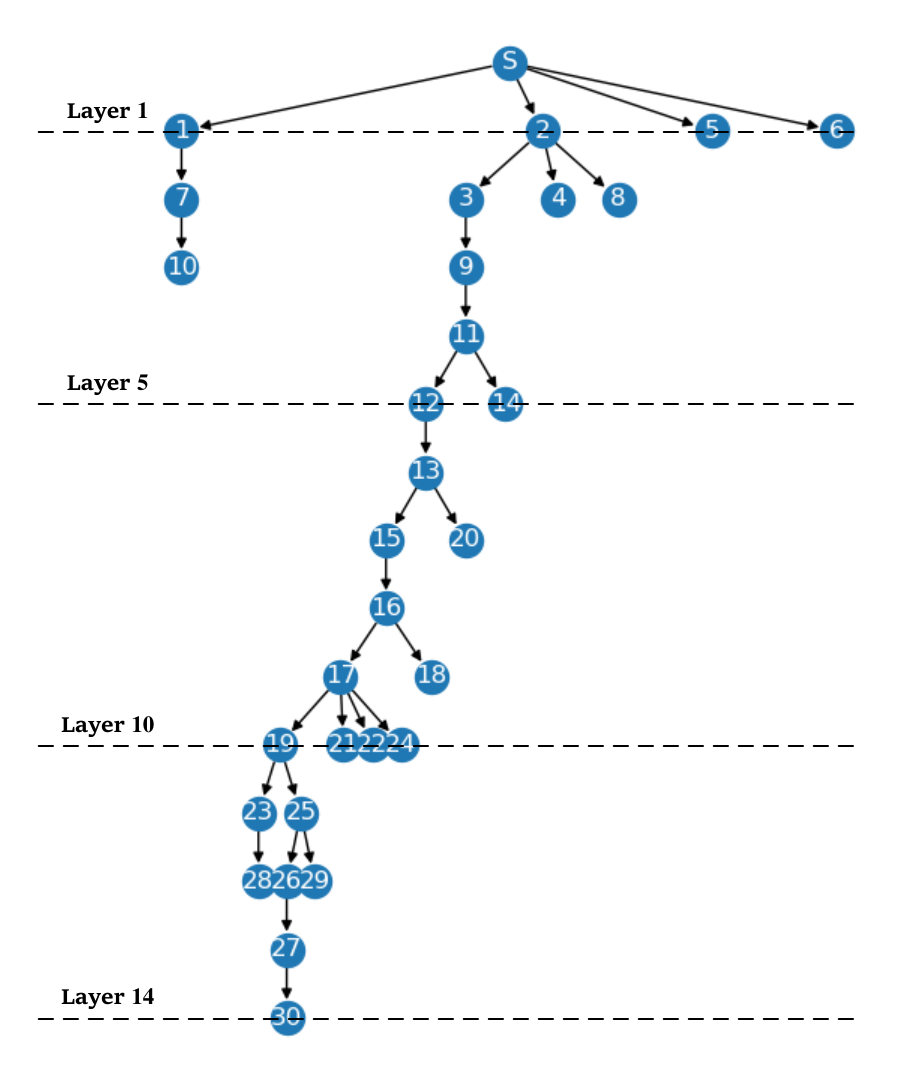}}
	\subfigure[iDFST Method Spanning Tree\label{fig:AlgorithmCompare_iDFST}]
	{\includegraphics[width=0.32\linewidth]{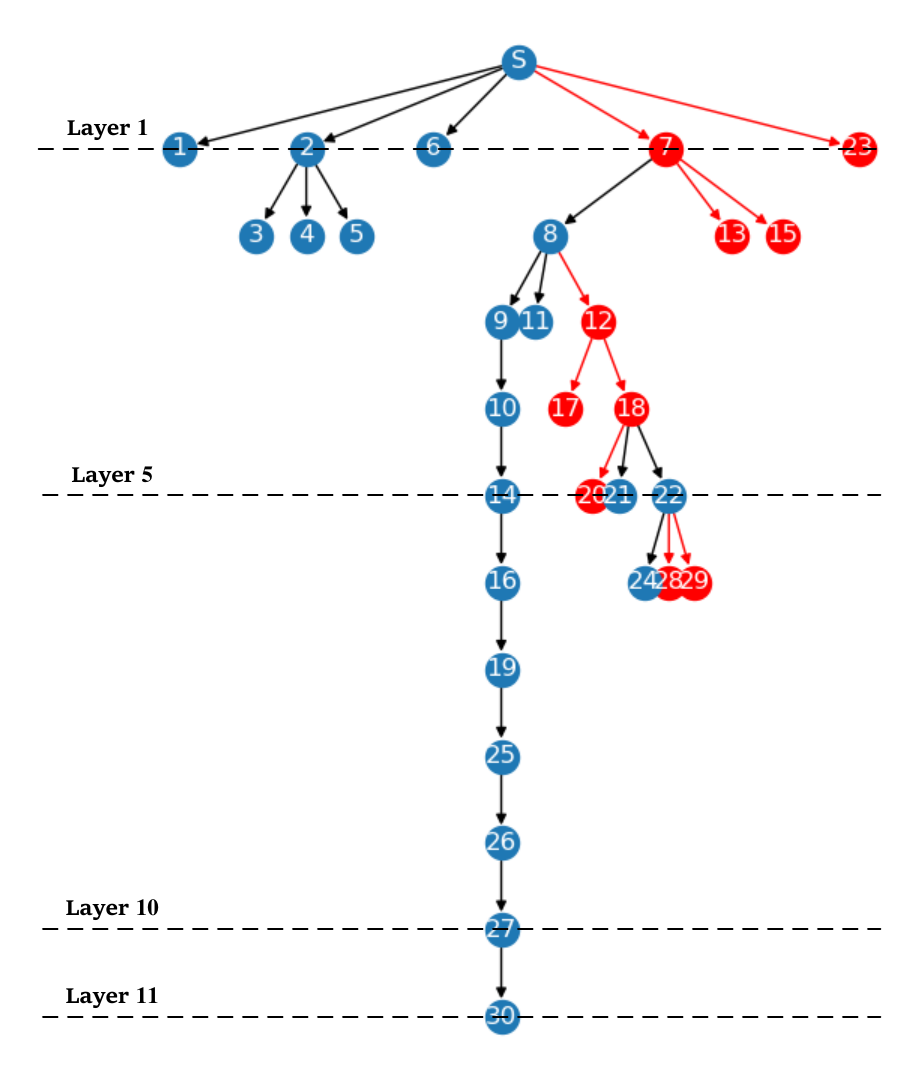}}
	\subfigure[MCC Method Spanning Tree\label{fig:AlgorithmCompare_MCC}]
	{\includegraphics[width=0.32\linewidth]{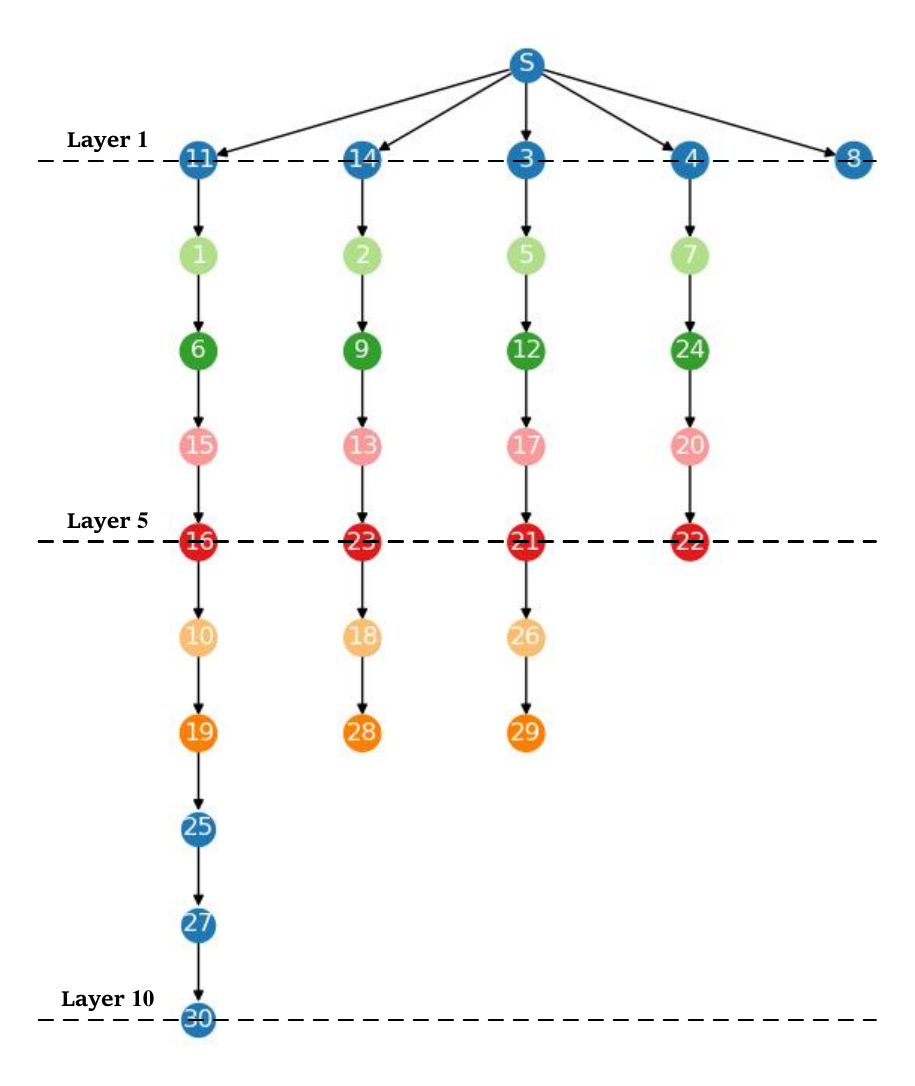}}
	\caption{A case study of $ 30 $ CAVs. CDG and CUG are presented in Fig.~\ref{fig:AlgorithmCompare_CDG} and Fig.~\ref{fig:AlgorithmCompare_CUG}, respectively, to describe the conflict and coexistence relationship of the CAVs. DFST arranges the CAVs as shown in Fig~.\ref{fig:AlgorithmCompare_DFST} with the maximum layer of $ d_\mathrm{all} = 14 $. iDFST, which is an improvement of DFST,  utilizes some nodes (red ones) to optimize the overall depth, as shown in Fig.~\ref{fig:AlgorithmCompare_iDFST}, resulting in $ d_\mathrm{all} = 11 $. The MCC method finds the minimum group of CAVs that can pass through the intersection simultaneously (nodes of the same color), resulting in $ d_\mathrm{all} = 10 $, as shown in Fig.~\ref{fig:AlgorithmCompare_MCC}. 
	\label{fig:AlgorithmCompare}}
\end{figure*}

We demonstrate a case study in Example~\ref{exp:2} to show the effectiveness of the proposed algorithms. The benchmark algorithm is the DFST method proposed in ~\cite{xu2018distributed}. Based on DFST, the first algorithm is the iDFST presented in Section~\ref{Sec:iDFST}. The MCC algorithm has been explained in Section~\ref{Sec:MCC}, which applies a heuristic algorithm to solve the MCC problem.

\begin{example}[A case study of $ 30 $ CAVs]
	\label{exp:2}
	The traffic scenario remains the same as shown in Fig~\ref{fig:Traffic_Scenario}. A total number of $ 30 $ vehicles arrive at the intersection with $ \lambda = 3 $ Poisson distribution. The vehicle arrivals are generated randomly and remain identical among the three algorithms.
\end{example}

As mentioned before, CDG and CUG are employed to describe the conflict relationship of the vehicles. The CDG is plotted in Fig.~\ref{fig:AlgorithmCompare_CDG} as explained in Definition~\ref{def:CDG}. The red dashed edges are the crossing conflicts and converging conflicts, whereas the black solid edges are the diverging and reachability conflicts. Fig.~\ref{fig:AlgorithmCompare_CUG} depicts the CUG defined in Definition~\ref{def:CUG}, which is the complement graph of the CDG. The edges are the coexisting relationships,~\ie, the conflict-free vehicles.

Although identical vehicle arrivals are set for the three algorithms, different spanning trees,~\ie, passing order solutions, are obtained. The spanning trees are carried out by scheduling the vehicles in the same layer to pass the intersection simultaneously. Thus, the overall depth $ d_\mathrm{all} $ of the spanning tree represents the evacuation time of the solution. First, the spanning trees of the DFST and iDFST methods are shown in Fig~\ref{fig:AlgorithmCompare_DFST} and Fig.~\ref{fig:AlgorithmCompare_iDFST}, which are derived from Fig.~\ref{fig:AlgorithmCompare_CDG}. Both the DFST and iDFST methods are based on the FIFO principle,~\ie, obtaining the spanning tree vehicle by vehicle. Because the DFST has limited optimization of vehicle scheduling, it is a feasible solution with $ d_\mathrm{all} = 14 $. In contrast, iDFST separates the conflict types to find the optimal place for each vehicle and obtains a locally optimal solution. Hence, the overall depth of the iDFST method is $ d_\mathrm{all} = 11 $. Note that the red nodes are the vehicles that are brought forward when compared with the DFST method. Fig~\ref{fig:AlgorithmCompare_MCC} depicts the solution derived from Fig.~\ref{fig:AlgorithmCompare_CUG} by using the MCC method. The nodes of the same color represent the vehicles in the same clique,~\ie, the vehicles that are scheduled to pass the intersection at the same time. The same cliques and the corresponding colors are drawn in Fig.~\ref{fig:AlgorithmCompare_CUG} to show the MCC result with $ d_\mathrm{all} = 10 $. {Moreover, it is observed that MCC generates a spanning tree much \textit{balanced} than iDFST. In the following simulations, we found that even if evacuation time and ATTD are not of much difference in iDFST and MCC, the distribution of the vehicle nodes influences the fuel consumption performance significantly.}


{
\begin{remark}
	Note that these algorithms run in a dynamic process in real-scenario implementations. When CAVs enter the control zone or reach the stopping line, CDG, CUG, and spanning trees in Fig.~\ref{fig:AlgorithmCompare} change dynamically. It means that Fig.~\ref{fig:AlgorithmCompare} is the intermediate state where the spanning trees grow into maximum depth $ d_\mathrm{all} $. For DFST/iDFST, CDG and spanning tree are calculated vehicle by vehicle. Therefore, as the CAVs gradually approach from $ 1 $ to $ N $, CDG and spanning tree gradually grow to Fig.~\ref{fig:AlgorithmCompare}. In other words, the arrival plan of $ k $ vehicles is calculated based on the arrival plan of $ k - 1 $ vehicles. For MCC, since the spanning tree is calculated in global optimal, the spanning tree of $ k $ vehicles is not directly related to the spanning tree of $ k-1 $ vehicles. In other words, the arrival plan is re-calculated as the new CAV comes. To further clarify the dynamic process, we upload the complete simulation GIF which is available at \url{https://github.com/CeroChen/GraphBashedCoordination}.
\end{remark}
}

\subsection{Algorithm Validation with Small Number of Vehicles}
\begin{figure*}[!t]
	\centering
	\subfigure[Maximum depth of the spanning tree generated by different algorithms\label{fig:NineVehicles:Depth}]
	{\includegraphics[width=0.95\linewidth]{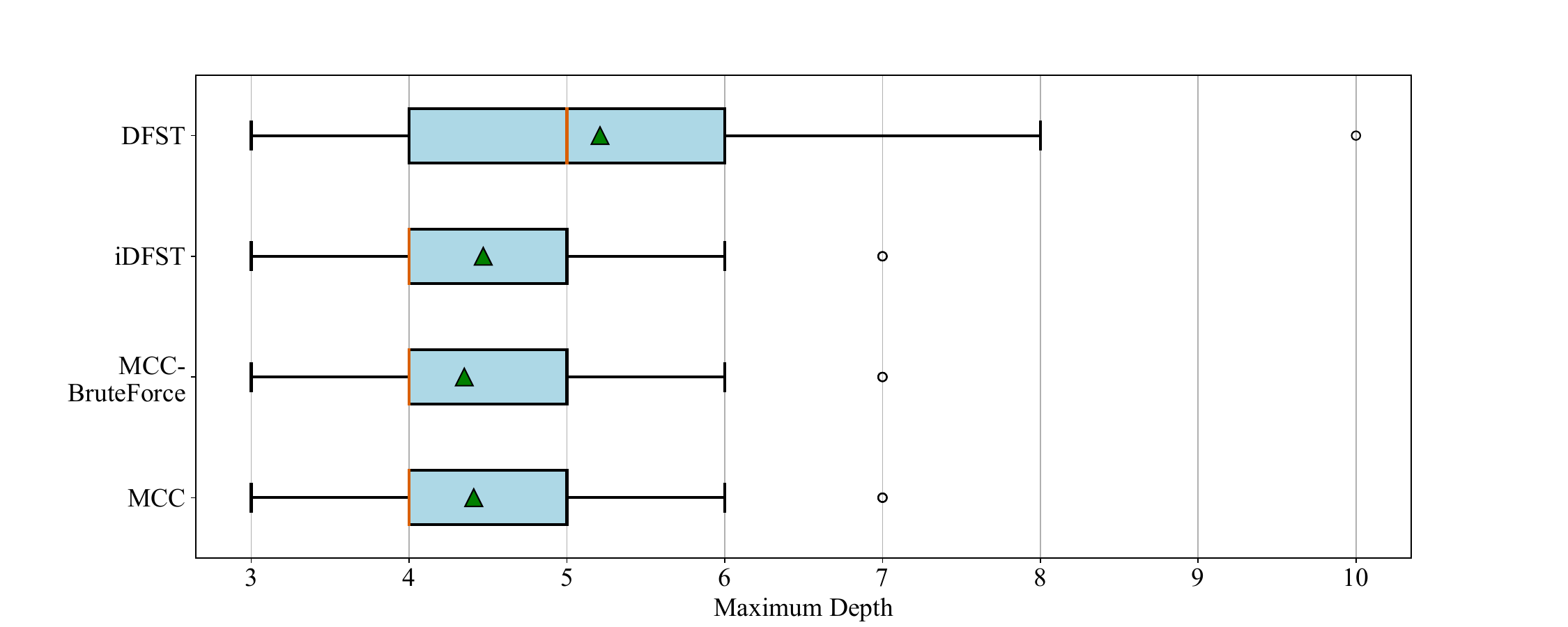}}
	
	\subfigure[Computational time of different algorithms\label{fig:NineVehicles:Time}]
	{\includegraphics[width=0.95\linewidth]{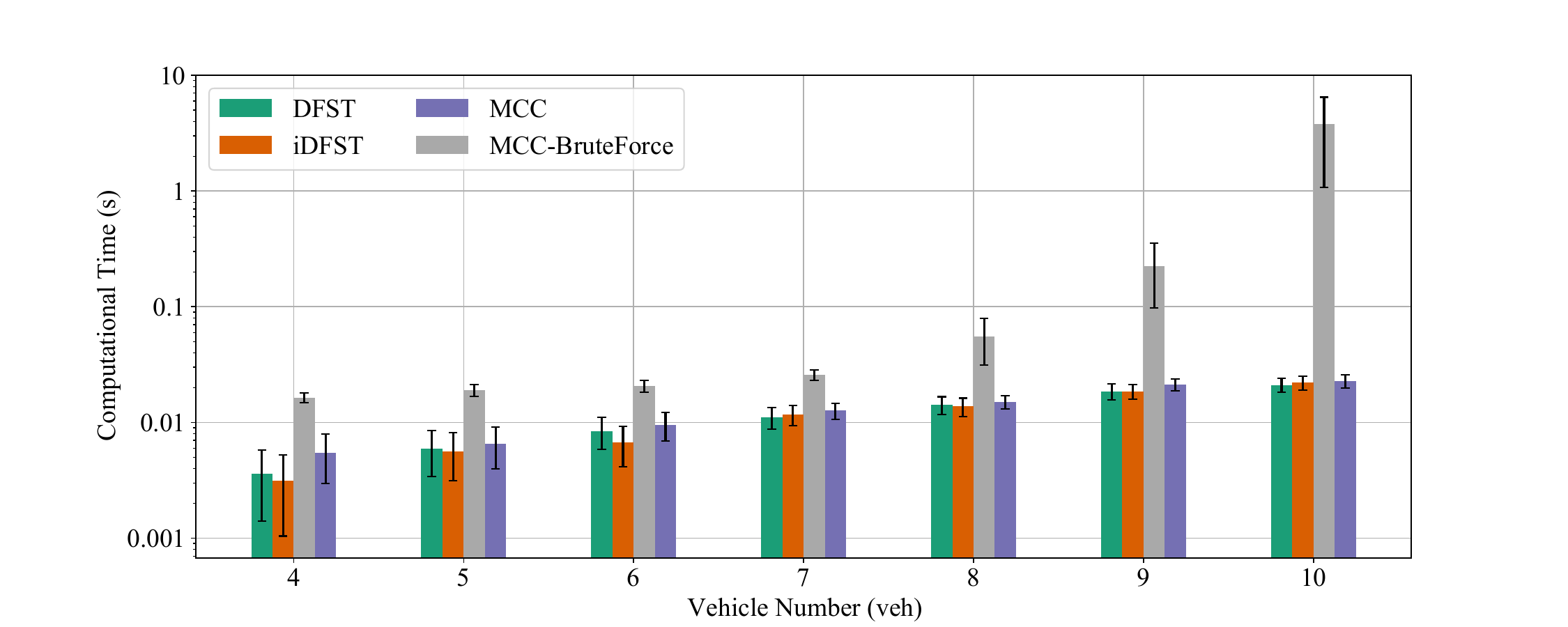}}
	\caption{The maximum depth of the spanning tree obtained by different algorithms is shown in Fig.~\ref{fig:NineVehicles:Depth}. Simulation is conducted $ 100 $ times for each algorithm and $ 10 $ vehicles are generated randomly in each simulation. The orange lines are the median values of the results and the green triangles are the average values. The corresponding computational time is shown in Fig.~\ref{fig:NineVehicles:Time}}
	\label{fig:NineVehicles}
\end{figure*}
Next, we focus on the optimality of the three algorithms in general cases. In this section, we introduce the MCC-BruteForce algorithm. It enumerates all the possible combinations of the cliques,~\ie, solves the MCC problem by using the brute force method. Although MCC-BruteForce provides us the theoretical global optimal solution, the huge calculation burden makes it impossible to apply it to a large number of vehicles. {In our simulation, MCC-BruteForce computation of $ 11 $ vehicles requires more than 32GB RAM, which exceeds the maximum RAM of our computer. Thus, $ 10 $ is set as the number of input vehicles.}

Because the vehicle scheduling problem is strongly related to the vehicle arrival time, conducting a small number of simulations is insufficient. Hence, $ 100 $ repetitions of the simulation were conducted for each algorithm. It is worth mentioning that these $ 100 $ instances of vehicle arrival input are randomly generated by Poisson distribution, as shown in~\eqref{equ:Poisson}, and they are identical across all four algorithms.

We use the box plot to illustrate the maximum depth of the spanning tree,~\ie, the overall traffic efficiency, as shown in Fig.~\ref{fig:NineVehicles:Depth}. The red lines are the median values of the corresponding simulation results, whereas the green triangles represent the average values. {Compared with the DFST method, both the iDFST and MCC algorithms can reduce the maximum depth. The first quartile value, median value, and third quartile value are $ 4,4,5 $, respectively. However, the average values of the results are not the same. The iDFST method attained an average value of $ 4.47 $, whereas the average value in the MCC method was $ 4.41 $ and that of MCC-BruteForce was $ 4.35 $. Recall that the vehicle arrivals were identical for these algorithms; therefore, it can be inferred that the MCC-BruteForce method obtains the global optimal solution. Because of the heuristic MCC method introduced in Section~\ref{Sec:MCC}, the MCC method cannot obtain the global optimal solution. Nevertheless, it achieves a better result than the iDFST method, which obtains the locally optimal solution. The computational time comparison is shown in Fig.~\ref{fig:NineVehicles:Time}. As vehicle number increases, the computational time of MCC-BruteForce method increases exponentially. For $ 10 $ CAVs, computational time exceeds $ 1 s$, which does not meet the real scenario deployment requirement. The standard deviation of MCC-BruteForce method also increases with vehicle number. It implies that the computational time of MCC-BruteForce method is strongly-related to vehicle distribution. By contrast, the computational time of the other three methods remains below $ 0.03 s$ and has a low standard deviation. The computational time of more vehicles is discussed in Section~\ref{Sec:SimulationResults:InputVehicles}.}



\subsection{Simulation Results for Various Numbers of Input Vehicles}
\label{Sec:SimulationResults:InputVehicles}
\begin{figure*}[!p]
	\centering
	\subfigure[Evacuation time\label{fig:VehicleNumber_Duration_New}]
	{\includegraphics[width=0.93\linewidth]{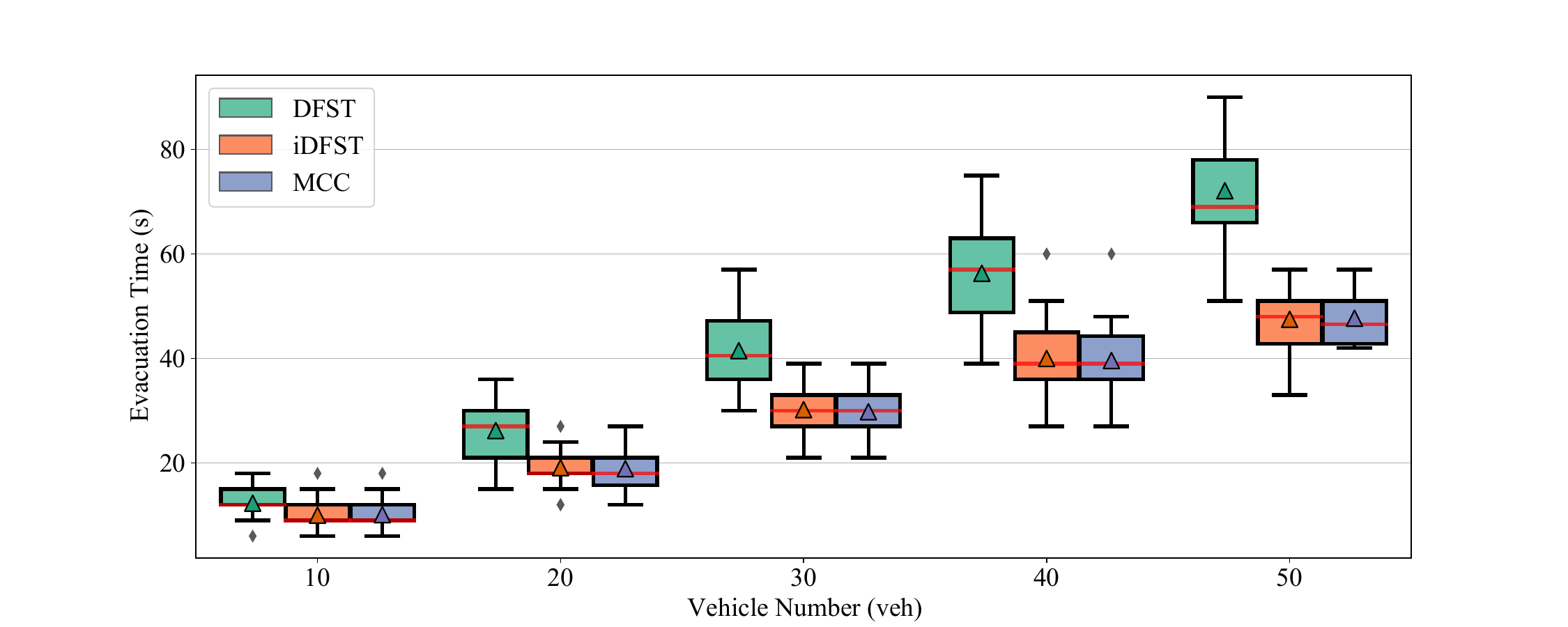}}
	\subfigure[Average travel time delay\label{fig:VehicleNumber_MeanTravelTime}]
	{\includegraphics[width=0.93\linewidth]{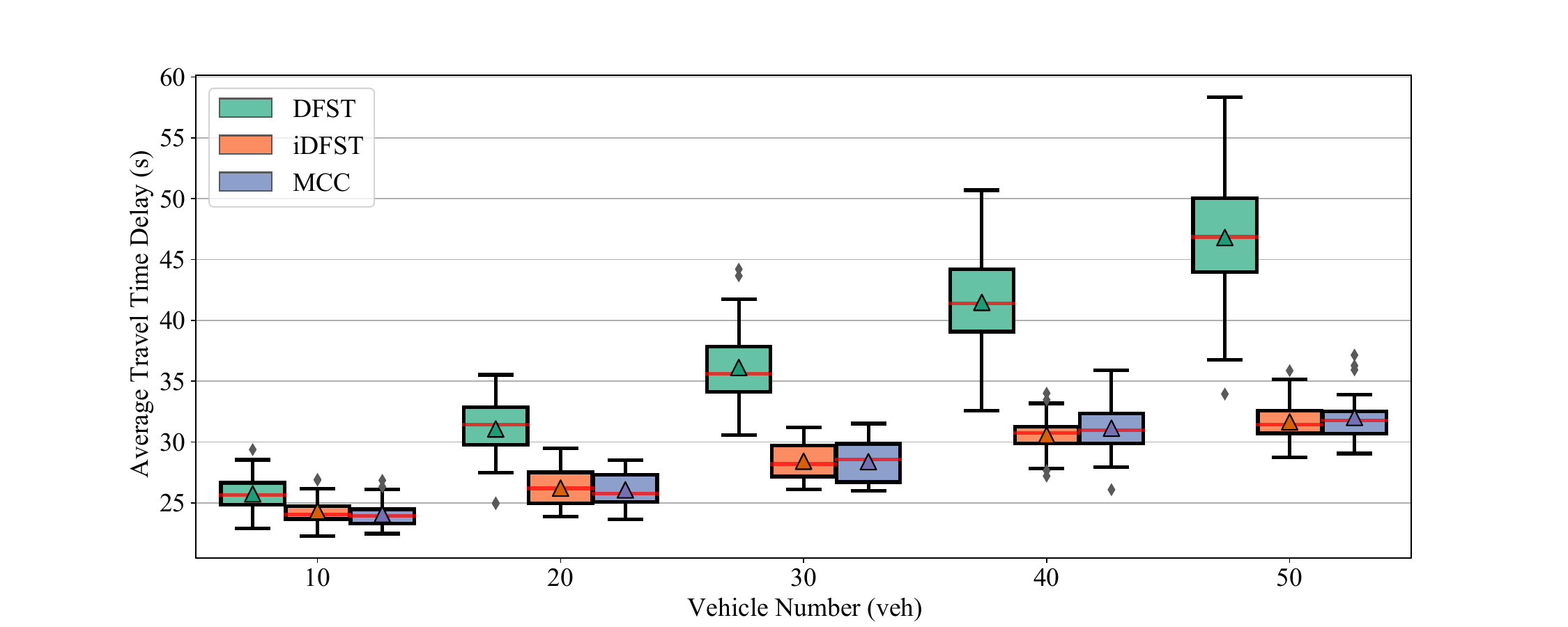}}
	\subfigure[Fuel consumption\label{fig:VehicleNumber_EnergyConsumption}]
	{\includegraphics[width=0.93\linewidth]{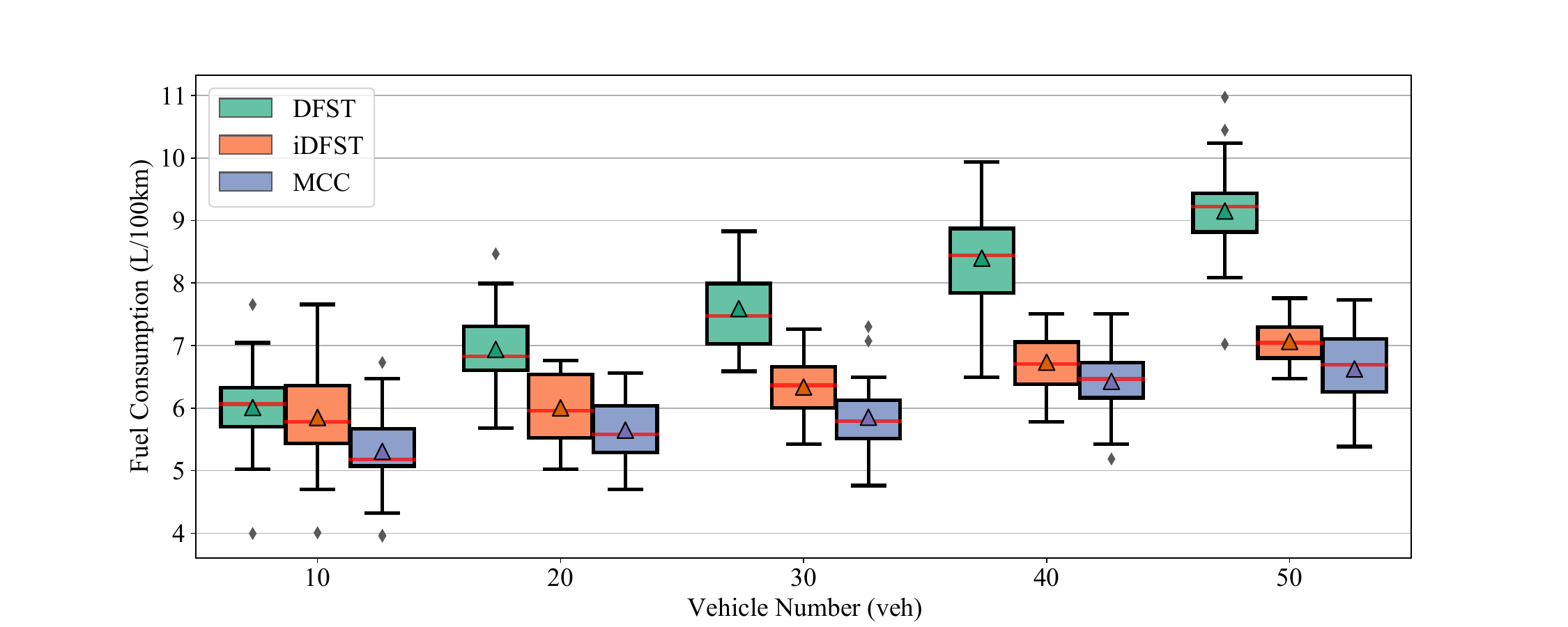}}
	\caption{Comparison of traffic efficiency and fuel consumption of three algorithms for various numbers of input vehicles. For each algorithm, ten repetitions of the simulation were conducted for each number of input vehicles. The red lines are the median values and the triangles are the average values. As the number of vehicles increases, MCC shows better performance than the other two algorithms in improving the traffic efficiency and reducing the fuel consumption.}
	\label{fig:VariousVehicleInput}
\end{figure*}

For conducting the simulation on a larger scale, we extended the number of input vehicles from $ 10 $ to $ 50 $ and further compared the algorithms. $ \lambda=3 $ was set as the input for the Poisson distribution of the vehicle arrival. As shown in Fig~\ref{fig:VariousVehicleInput}, the red lines in the box plot are the median value and the triangles are the average values. Ten repetitions of the simulation were conducted for each number of vehicles, and the vehicle arrivals were identical across the three algorithms.

In terms of the overall traffic efficiency performance, both the iDFST and MCC methods reduce the evacuation time significantly when compared with DFST. As shown in Fig~\ref{fig:VehicleNumber_Duration_New}, the average value of the evacuation time in DFST for the number of input vehicles of $ 50 $ is $ 69.9 \, s $, whereas those in iDFST and MCC are $ 46.2 \, s $ and $ 46.05 \, s $, respectively. The two proposed algorithms save approximately $ 33\% $ of the evacuation time because of the improved scheduling method. Similar results are observed in the ATTD in Fig.~\ref{fig:VehicleNumber_MeanTravelTime}, where iDFST and MCC save approximately $ 18 \% $ of the ATTD for $ 50 $ input vehicles. Although the MCC method prioritizes the overall efficiency,~\ie, the evacuation time in the scheduling, the re-ordering of the cliques minimizes the ATTD (line~\ref{algo:MCC:Spanning} in the MCC algorithm). In other words, the MCC method heuristically finds the MCC groups to minimize the evacuation time. Then, it obtains a spanning tree from the MCC groups to minimize the ATTD. {In MCC, the vehicle is scheduled to its near-globally optimal position which is not necessary to be the best position for itself (layer as smallest as possible). Therefore, more vehicles in MCC run in smooth speed trajectories instead of a harsh acceleration/deceleration in driving into their own best position in iDFST. As shown in Fig.~\ref{fig:VehicleNumber_EnergyConsumption}, the better vehicle distribution helps to avoid the deceleration and idling of the CAVs in the control zone. iDFST and MCC saves $ 20 \% $ and $ 27 \% $ of the fuel consumption, respectively, for $ 50 $ input vehicles.}


{
In terms of computational time, the calculation time of different vehicle numbers is summed up in Table~\ref{tab:ComputationalTime}. In the comparison of the three algorithms, the computational time of MCC is slightly larger (1\%-2\%) than the other two algorithms. We infer that this is because of the constant time graph-based procedures in MCC algorithm. We have interpreted in Section~\ref{Sec:iDFST} and~\ref{Sec:MCC} that all of these algorithms have computational complexity of $ O(N) $. The simulation results validate this conclusion, where computational time approximately linearly increases with vehicle number.
}
\begin{table}[!t]
	\caption{Average value of computational time (seconds) comparison.}
	\label{tab:ComputationalTime}
	\renewcommand{\arraystretch}{1.3}
	\centering
	\begin{tabular}{c|c|c|c|c|c}
		\hline
		\diagbox[width=10em]{\textbf{Algorithm}}{\textbf{Vehicle} \\ \textbf{Number}} & $ \mathbf{10} $ & $ \mathbf{20} $  & $ \mathbf{30} $ & $ \mathbf{40} $ & $ \mathbf{50} $ \\
		\hline
		DFST & $ 0.022 $ & $ 0.092 $ & $ 0.209 $ & $ 0.369 $ & $ 0.575 $\\
		iDFST & $ 0.023 $ & $ 0.094 $ & $ 0.211 $ & $ 0.377 $ & $ 0.578 $\\
		MCC & $ 0.025 $ & $ 0.097 $ & $ 0.217 $ & $ 0.386 $ & $ 0.590 $\\
		\hline
	\end{tabular}
\end{table}

\begin{table}[!t]
	\caption{Standard deviation of computational time (seconds) comparison.}
	\label{tab:ComputationalTime_std}
	\renewcommand{\arraystretch}{1.3}
	\centering
	\begin{tabular}{c|c|c|c|c|c}
		\hline
		\diagbox[width=10em]{\textbf{Algorithm}}{\textbf{Vehicle} \\ \textbf{Number}} & $ \mathbf{10} $ & $ \mathbf{20} $  & $ \mathbf{30} $ & $ \mathbf{40} $ & $ \mathbf{50} $ \\
		\hline
		DFST & $ 0.007 $ & $ 0.012 $ & $ 0.021 $ & $ 0.031 $ & $ 0.048 $\\
		iDFST & $ 0.008 $ & $ 0.012 $ & $ 0.019 $ & $ 0.031 $ & $ 0.048 $\\
		MCC & $ 0.008 $ & $ 0.012 $ & $ 0.018 $ & $ 0.027 $ & $ 0.045 $\\
		\hline
	\end{tabular}
\end{table}

\subsection{Simulation Results at Various Traffic Volumes}

\begin{figure*}[!p]
	\centering
	\subfigure[Evacuation time\label{fig:Lambda_Duration_New}]
	{\includegraphics[width=0.93\linewidth]{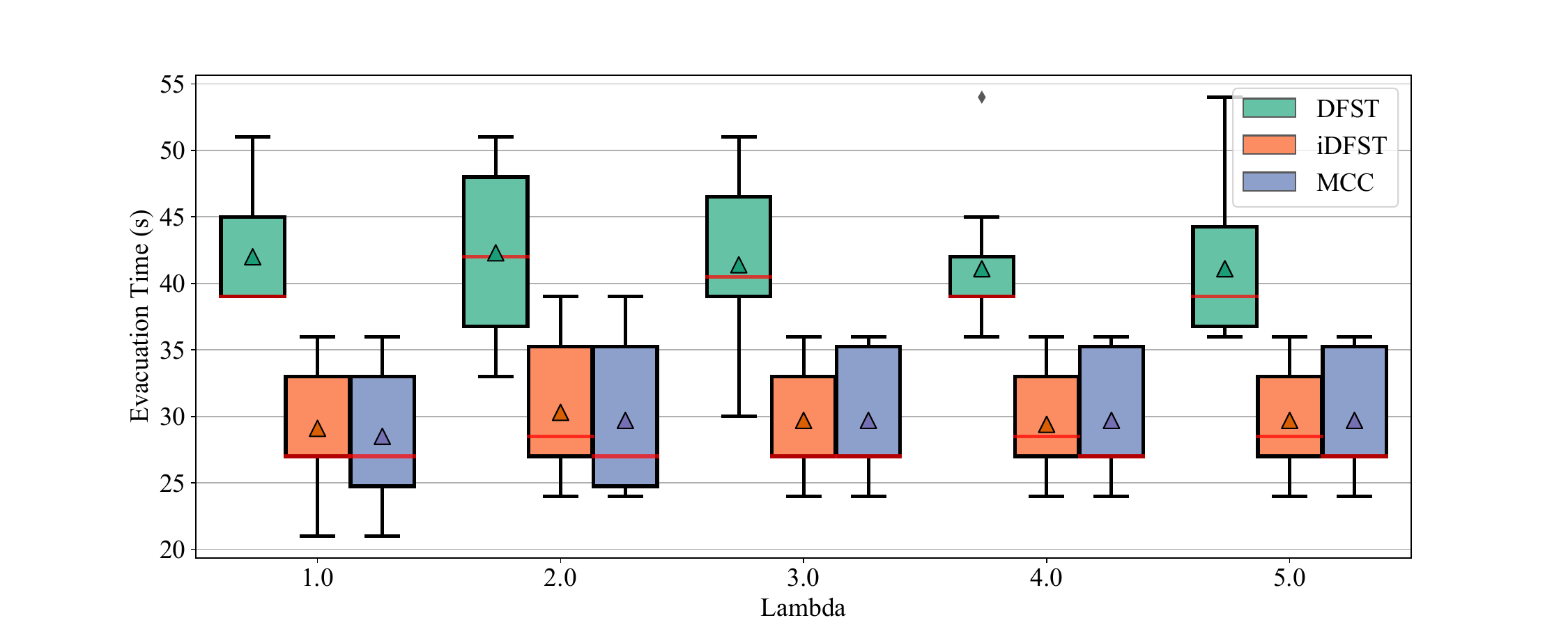}}
	\subfigure[Average travel time delay\label{fig:Lambda_MeanTravelTime}]
	{\includegraphics[width=0.93\linewidth]{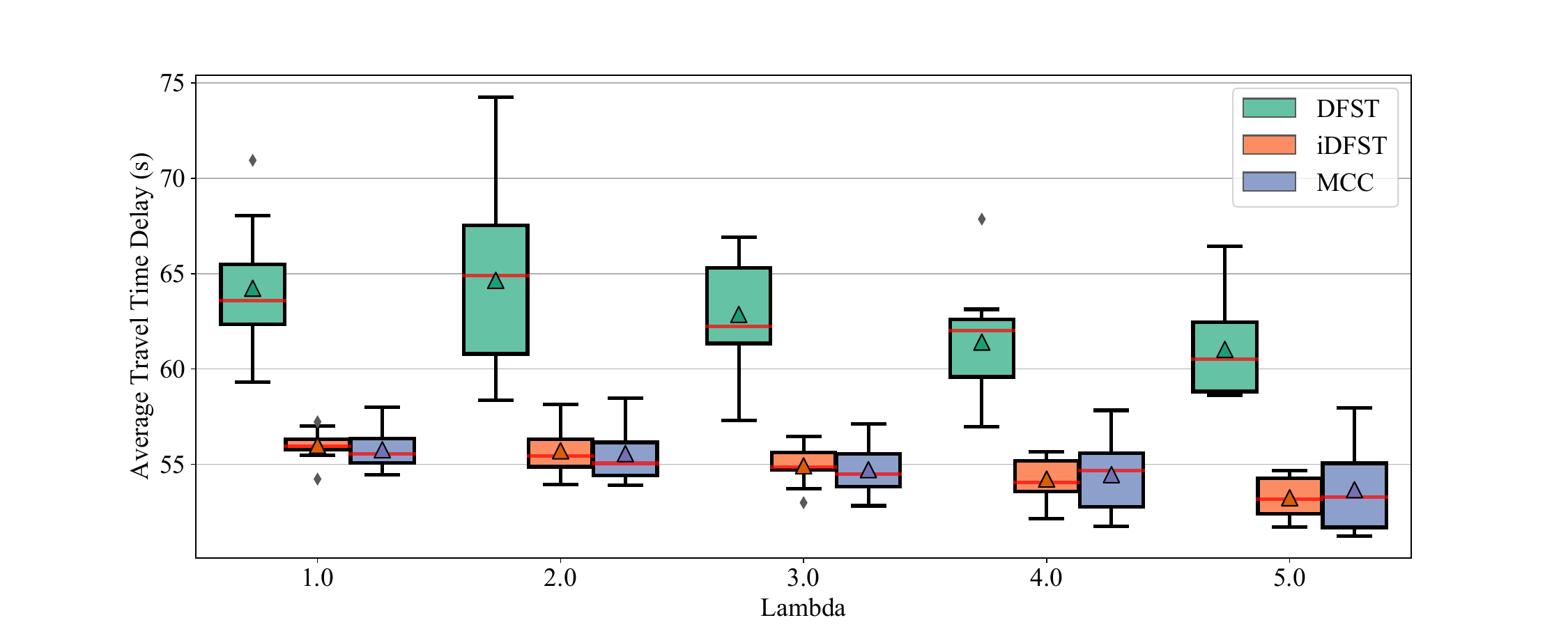}}
	\subfigure[Fuel consumption\label{fig:Lambda_EnergyConsumption}]
	{\includegraphics[width=0.93\linewidth]{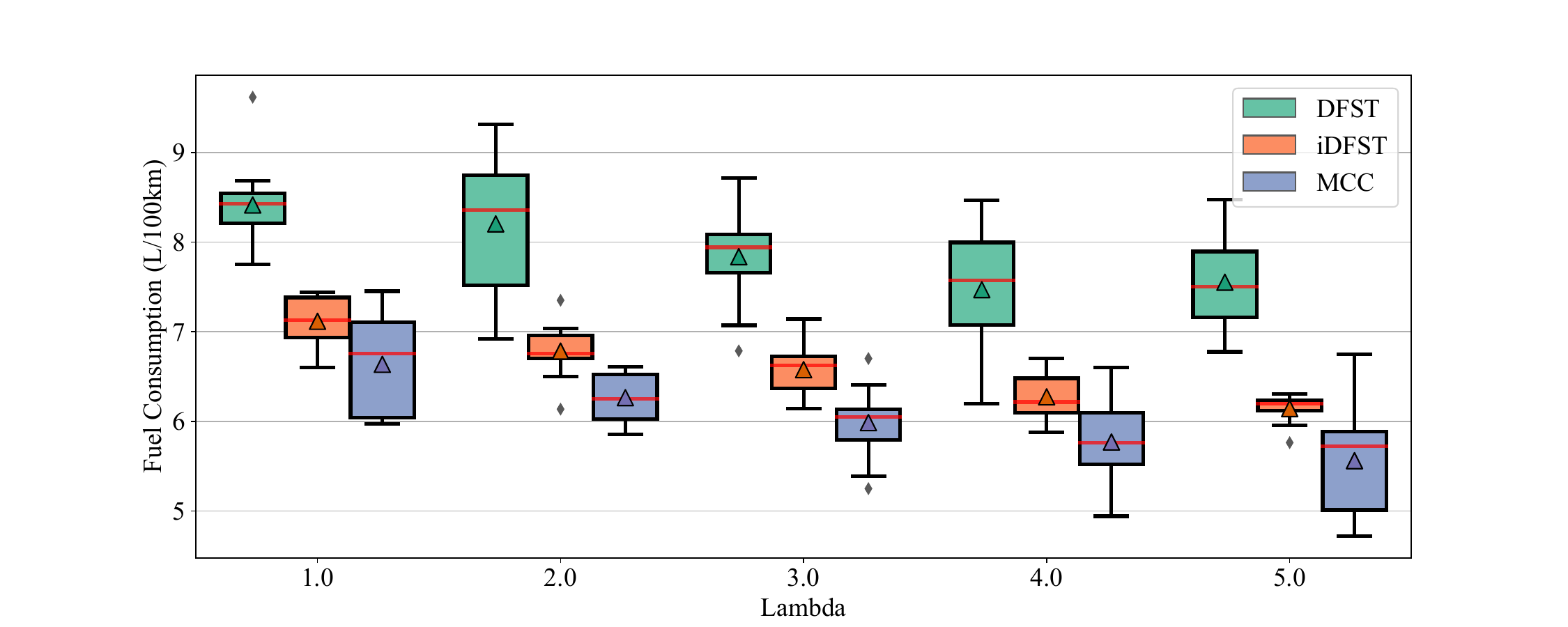}}
	\caption{Comparison of traffic efficiency and fuel consumption of three algorithms for various vehicle volumes. Ten repetitions of the simulation were conducted for each vehicle volume number. The red lines are the median values and the triangles are the average values. MCC shows better performance than the other two algorithms, especially when the intersection is rather crowded,~\ie, $ \lambda < 3 $.}
	\label{fig:VariousLambda}
\end{figure*}

The previous simulations were conducted with $ \lambda=3 $,~\ie, a constant Poisson distribution. In the last simulation, we examined the influence of the traffic volume on the algorithm performance. Because the vehicle arrival follows the Poisson distribution as shown in~\eqref{equ:Poisson}, we varied the $ \lambda $ value to change the traffic volumes on each lane. As before, $ 10 $ repetitions of the simulation were conducted for each number of vehicles, and the vehicle arrivals were identical across the three algorithms. The simulation results are plotted in Fig~\ref{fig:VariousLambda}.

$ \lambda $ is the expected value of the Poisson distribution, which represents the average gap in the vehicle arrival time on each lane. As shown in Fig.~\ref{fig:Lambda_Duration_New}, the evacuation times obtained with the three algorithms show limited fluctuations with increase in $ \lambda $. The proposed two algorithms,~\ie, iDFST and MCC, reduce the evacuation time obtained with DFST by approximately $ 30 \% $. Note that MCC has better performance when $ \lambda < 3 $. Recall that in Table~\ref{tab:Parameters}, we designed the car following distance of the virtual platoon as $ D_\mathrm{des} = 30 \, m $ and the speed of the virtual leading vehicle as $ v_\mathrm{p} = 10 \, m/s $. Thus, the theoretical vehicle output of the intersection is bounded as $ {D_\mathrm{des}}/{v_\mathrm{p}} = 3 \, s $ per layer. In other words, if the average input vehicle gap $ \lambda < 3 $, the accumulation and queuing of the vehicles is inevitable. Under this circumstance, MCC shows better performance than iDFST because of its better usage of the intersection space. Fig.~\ref{fig:Lambda_MeanTravelTime} shows the ATTD performance of the algorithms. As the traffic volume increases, the vehicle is more likely to be stuck behind other preceding vehicles. Thus, the vehicle travel time is positively correlated with the traffic volume,~\ie, the ATTD decreases as the vehicle arrival gap increases. Similar to the evacuation time, MCC shows better performance when $ \lambda < 3 $. Because of the aforementioned reason for the MCC to show better utilization of the intersection space, idling behavior and fuel consumption performance are greatly improved. As shown in Fig.~\ref{fig:Lambda_EnergyConsumption}, iDFST saves at most $ 15 \% $ and MCC saves $ 26 \% $ of the fuel consumption when compared with DFST.


\section{Conclusion}
\label{Sec:Conclusion}
In this paper, a graph-based cooperation method was proposed to formulate the conflict-free scheduling problem at unsignalized intersections. {Based on the graphical description of the conflict relationships among the CAVs, an iDFST method was introduced to obtain the local optimal solution. After concise and rigid graphical analysis, we reduced the CAV scheduling problem to the MCC problem, which yielded the global optimal solution. MCC problem was proven to be an NP-hard problem, therefore brute-force method to solve MCC is unrealistic in real-world deployments. To solve the problem, a heuristic method was proposed to solve MCC problem with low computation complexity in the case of a large number of vehicles. Furthermore, a distributed control framework and communication topology were designed to realize the conflict-free cooperation of the vehicles. Traffic simulations proved the effectiveness of the proposed algorithms, where iDFST and MCC algorithms outperform the original DFST method. In the trade-off between scheduling optimality and computation efficiency, although iDFST and MCC cannot guarantee the global optimal solution, they generate rather good CAV arrival plans and have a low computational complexity of $ O(N) $. Thus, both of them are qualified to schedule dozens of CAVs in real intersection scenario deployment. Note that these proposed scheduling algorithms can be applied to intersections with any geometrical structures, which proves the generalization ability of our algorithm.}

A future direction of our research is to permit the lane-changing behavior,~\ie, the CAVs are allowed to change lanes while approaching the intersection. Several practical methods~\cite{cai2021formation, chouhan2020cooperative} have been proposed to apply the lane changing behavior of multiple CAVs. It is interesting to combine these works to further extend the single intersection into multiple intersection cooperation, which will be considered in our future research.  Another interesting topic is the communication topology of the virtual platoon. In this study, we designed a PLF communication topology to realize the scheduling plan. The problem of identifying the best topology is yet to be addressed. {Finally, field experiments are needed for algorithm validation, where more realistic vehicle dynamic model and vehicle trajectory planning in the conflict zone should be further considered.}

%
%

\ifCLASSOPTIONcaptionsoff
  \newpage
\fi



\bibliographystyle{IEEEtran}
\bibliography{IEEEabrv,mybibfile}
%
%
%

%

\begin{IEEEbiography}[{\includegraphics[width=1in,height=1.25in,clip,keepaspectratio]{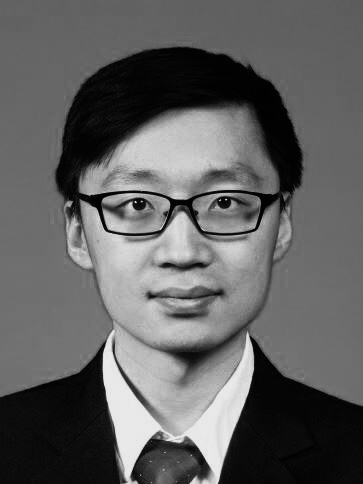}}]
	{Chaoyi Chen} (Graduate Student Member, IEEE) received the B.E. degree from Tsinghua University, Beijing, China, in 2016, and the M.S. from Tsinghua University, Beijing, China, and RWTH Aachen University, Aachen, Germany in 2019. He is currently a Ph.D. student in mechanical engineering with the School of Vehicle and Mobility, Tsinghua University. He was a recipient of the Scholarship of Strategic Partnership RWTH Aachen University and Tsinghua University. His research interests include vehicular networks, control theory, and cooperative control.
\end{IEEEbiography}


\begin{IEEEbiography}[{\includegraphics[width=1in,height=1.25in,clip,keepaspectratio]{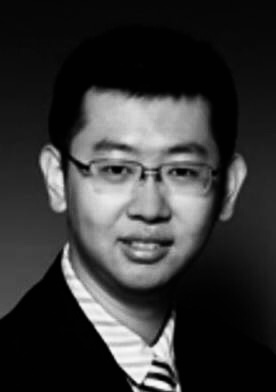}}]
	{Qing Xu} received his B.S. and M.S. degrees in automotive engineering from Beihang University, Beijing, China, in 2006 and 2008, respectively, and the Ph.D. degree in automotive engineering from Beihang University in 2014.
	
	During his Ph.D. research, he worked as a Visiting Scholar with the Department of Mechanical Science and Engineering, University of Illinois at Urbana–Champaign. From 2014 to 2016, he had his postdoctoral research at Tsinghua University. He is currently working as an Assistant Research Professor with the School of Vehicle and Mobility, Tsinghua University. His main research interests include decision and control of intelligent vehicles.
\end{IEEEbiography}


\begin{IEEEbiography}[{\includegraphics[width=1in,height=1.25in,clip,keepaspectratio]{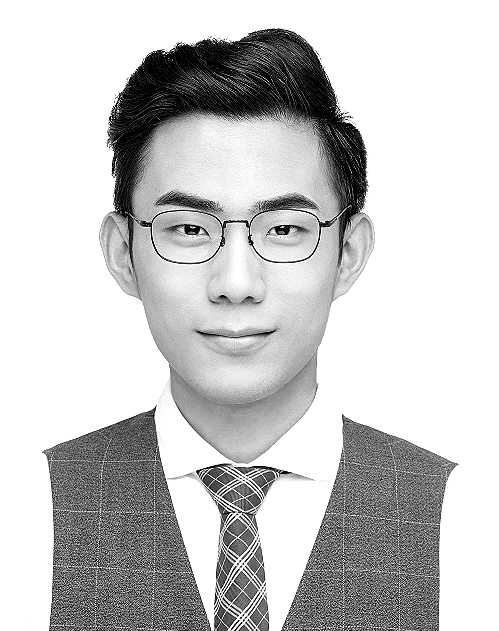}}]
	{Mengchi Cai} (Graduate Student Member, IEEE) received the B.E. degree from Tsinghua University, Beijing, China, in 2018. He is currently a Ph.D. candidate in mechanical engineering with the School of Vehicle and Mobility, Tsinghua University. His research interests include connected and automated vehicles, multi-vehicle formation control, and unsignalized intersection cooperation. 
\end{IEEEbiography}


\begin{IEEEbiography}[{\includegraphics[width=1in,height=1.25in,clip,keepaspectratio]{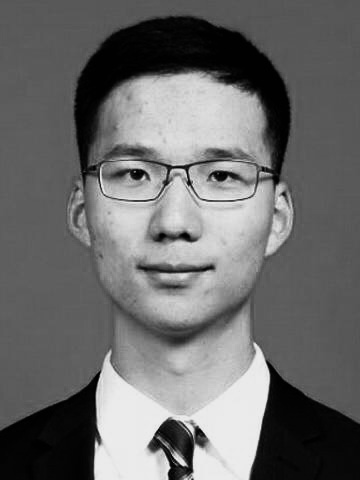}}]{Jiawei Wang} (Graduate Student Member, IEEE) received the B.E. degree from Tsinghua University, Beijing, China, in 2018. He is currently a Ph.D. student in mechanical engineering with the School of Vehicle and Mobility, Tsinghua University. His research interests include connected automated vehicles, distributed control and optimization, and data-driven control. He was a recipient of the National Scholarship in Tsinghua University. He received the Best Paper Award at the 18th COTA International Conference of Transportation Professionals. 
\end{IEEEbiography}


\begin{IEEEbiography}[{\includegraphics[width=1in,height=1.25in,clip,keepaspectratio]{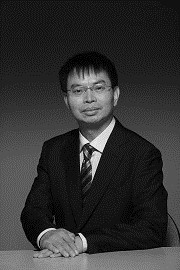}}]{Jianqiang Wang} received the B. Tech. and M.S. degrees from Jilin University of Technology, Changchun, China, in 1994 and 1997, respectively, and the Ph.D. degree from Jilin University, Changchun, in 2002. He is currently a Professor with the School of Vehicle and Mobility, Tsinghua University, Beijing, China. 
	
He has authored over 150 papers and is a co-inventor of 99 patent applications. He was involved in over 10 sponsored projects. His active research interests include intelligent vehicles, driving assistance systems, and driver behavior. He was a recipient of the Best Paper Award in the 2014 IEEE Intelligent Vehicle Symposium, the Best Paper Award in the 14th ITS Asia Pacific Forum, the Best Paper Award in the 2017 IEEE Intelligent Vehicle Symposium, the Changjiang Scholar Program Professor in 2017, the Distinguished Young Scientists of NSF China in 2016, and the New Century Excellent Talents in 2008.
\end{IEEEbiography}

\begin{IEEEbiography}[{\includegraphics[width=1in,height=1.25in,clip,keepaspectratio]{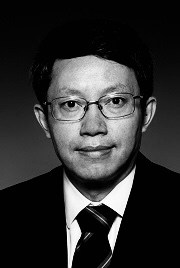}}]{Keqiang Li} received the B.Tech. degree from Tsinghua University of China, Beijing, China, in 1985, and the M.S. and Ph.D. degrees in mechanical engineering from the Chongqing University of China, Chongqing, China, in 1988 and 1995, respectively.
	
He is currently a Professor with the School of Vehicle and Mobility, Tsinghua University. His main research areas include automotive control system, driver assistance system, and networked dynamics and control, and is leading the national key project on CAVs (Intelligent and Connected Vehicles) in China. Dr. Li has authored more than 200 papers and is a co-inventor of over 80 patents in China and Japan.

Dr. Li has served as a Fellow Member of Society of Automotive Engineers of China, editorial boards of the \emph{International Journal of Vehicle Autonomous Systems}, Chairperson of Expert Committee of the China Industrial Technology Innovation Strategic Alliance for CAVs (CACAV), and CTO of China CAV Research Institute Company Ltd. (CCAV). He has been a recipient of Changjiang Scholar Program Professor, National Award for Technological Invention in China, etc.
\end{IEEEbiography}




\end{document}